\documentclass[11pt]{article}  

\usepackage{amssymb}
\usepackage[table]{xcolor}
\usepackage{amssymb,amsmath,amsthm}
\usepackage{symbols}
\usepackage{natbib}
\usepackage{tikz}
\usepackage{subcaption}
\usepackage{cleveref}
\usepackage{float}
\usepackage{stmaryrd}
\usepackage{etoolbox}
\usepackage{symbols}

\usepackage{dm-colors}
\usepackage[decisionutilitycolor]{influence-diagrams}
\usepackage{thmtools}
\usepackage{tabularx}
\usepackage{booktabs}
\usepackage{xspace}
\usepackage{enumitem}
\usepackage{algorithm}
\usepackage{algpseudocode}
\usepackage{wrapfig}
\usepackage[letterpaper,left=1.7in,right=1.7in,top=1.45in,bottom=1.52in]{geometry}

\usepackage[font=small,labelfont=bf]{caption}

\usetikzlibrary{positioning}

\let\savedegree\degree 
\let\degree\relax
\usepackage{mathabx}
\let\degree\savedegree

\newif\ifthesis
\thesisfalse  

\newtheorem{definition}{Definition}
\newtheorem{theorem}{Theorem}
\newtheorem{lemma}{Lemma}

\newtheorem{example}{Example}
\newtheorem{proposition}{Proposition}
\newcommand*{\ryan}[1]{\textcolor{brown}{}}
\newcommand*{\chris}[1]{\textcolor{blue}{}}
\newcommand*{\Chris}[1]{\textcolor{blue}{}}
\newcommand*{\tom}[1]{\textcolor{green}{}}

\usepackage{setup}

\title{Incentives for Responsiveness, Instrumental Control and Impact}
\author{%
  Ryan Carey\\
  University of Oxford
  \and
  Eric Langlois\\
  University of Toronto
  \and
  Chris van Merwijk\\
  Carnegie Mellon University
  \and
  Shane Legg\\
  Google DeepMind
  \and
  Tom Everitt\\
  Google DeepMind
}
\date{}

\begin{document}
\maketitle

\begin{abstract}

We introduce three 
concepts that describe an agent's incentives: 
response incentives
indicate which variables in the environment, 
such as sensitive demographic information, 
affect the decision under the optimal policy.
Instrumental control incentives 
indicate whether an agent's policy is chosen to manipulate part of its environment, 
such as the preferences or instructions of a user. 
Impact incentives 
indicate which variables an agent will affect, intentionally or otherwise.
For each concept, we establish sound and complete graphical criteria, 
and discuss general classes of techniques that may be used 
to produce incentives for safe and fair agent behaviour.
Finally, we outline how these notions may be generalised to multi-decision settings.

This journal-length paper extends our conference publication ``Agent Incentives: A Causal Perspective'': the material on response incentives and instrumental control incentives is updated, while the work on impact incentives and multi-decision settings is entirely new.
\end{abstract}

\section{Introduction}\label{sec:ai-acp:introduction}
In order to understand whether or not it is in your interests to interact with another agent, 
it is useful to consider that agent's incentives. 
In AI safety, for example, it has been argued that advanced AI systems would have an incentive to 
accumulate resources and/or to avoid being shut down \citep{omohundro2008basic,soares2015corrigibility}.
Such motives have been termed \emph{convergent instrumental goals}, because it is imagined that they might 
help a wide range of agents to achieve their goals.

The notion of a convergent instrumental goal has not been formally defined, however,
and it is not immediately clear how 
an agent's convergent instrumental goals should relate to its intent 
or incentives.

Ideally, we would like to have some language to describe 
the incentives of AI systems, that allows us to judge whether 
those incentives will lead to safe or fair behaviour.
There does already exist some language for describing safe or 
fair behaviour directly, for instance \emph{counterfactual harm} \citep{mueller2023personalized,richens2022counterfactual} and 
\emph{counterfactual fairness} \citep{kusner2017counterfactual}.
There also exists language that is at least related to incentives.
A variable is said to have \emph{positive value of information} if 
knowledge of its assignment can improve expected utility, 
and \emph{positive value of control} if deciding its assignment can 
do the same.
These concepts, however, do not directly allow us to assess whether 
an agent will behave in a safe or fair manner.
In the present work, therefore, we seek to devise some incentive concepts that:
\begin{itemize}
    \item make predictions about whether unsafe or unfair behaviour will occur, and
    \item describe how optimal behaviour is decided.
\end{itemize}
In the process, we hope to clarify the idea of an agent's convergent instrumental goals, 
and to contrast this with previous definitions of intentional influence 
of a variable.

In order for incentive concepts to be applicable, we need a way to deduce whether they are present or not.
In some cases, it is possible to rule out the presence of some incentive using the graphical structure alone.
For instance, in the graph $X \to D \to U$ where 
$X$ is a chance event, $D$ is a decision and $U$ is a utility function, 
we can tell that $X$ has zero value of information, because it is independent of $U$ given $D$.
A criterion for making such evaluations is called a \emph{graphical criterion}.
So, for each incentive concept that we introduce, we will establish a graphical criterion, 
and will discuss how it could be applied to ensure safer AI behaviour.

One might wonder, although our main application area in this paper is AI safety, might these incentive concepts 
be equally applicable to the behaviour of human individuals, or other agents?
In fact, none of these concepts are specific to AI but they may be more naturally applicable to 
AI systems insofar as they are trained to pursue closed-form objective functions, 
whereas this is a looser approximation of human behaviour.

\paragraph{Overview of Contributions}
This paper will begin with some setup (\cref{sec:ai-acp:setup}).

Next, we will focus on the information an agent can benefit from \emph{using}, to make a decision.
In previous work, \emph{materiality} has described which actual observations aid performance \citep{shachter2016decisions}. 
In \cref{sec:ai-acp:materiality}, we 
prove
 a known graphical criterion that can be used to deduce, in some circumstances, 
that a variable is immaterial \citep{fagiuoli1998note,lauritzen2001representing}.
We prove that this criterion is complete, in that it proves 
immateriality 
whenever possible to do so from the graphical structure alone.

We then present a new concept, the \emph{response incentive} (RI) (\cref{sec:ai-acp:response}),
which describes which variables an agent's decision is influenced by, be they observed or causally upstream of the observations.
This is important to AI fairness, because it describes when an optimal agent will be counterfactually unfair \citep{kusner2017counterfactual}, 
and to AI safety, in that it relates to the obedience of an agent \citep{Hadfield-Menell2016osg,carey2023human}.
We also prove a simple graphical criterion that is sound and complete for ruling out an RI.

Next, we consider what variables an agent can benefit from \emph{influencing}.
The notion of \emph{value of control} \citep{shachter1986evaluating} describes what variables an agent would like to control, but it falls short 
in describing what variables an agent is actually incentivized to control.
So we introduce a new concept, the \emph{instrumental control incentive} (ICI), 
which describes variables that an agent has both a need, and a means to influence 
(\cref{sec:ai-acp:fci}).
The instrumental control incentive attempts to formalize the notion of an \emph{instrumental goal} in AI safety, 
and the idea that an agent is incentivized to ``try'' to influence some variable.
We demonstrate that it is closely related to the notion of \emph{intent}, from \citet{halpern2018towards,ward2024reasons}, 
and prove an identical sound and complete criterion for each of these concepts
(\cref{sec:intent}).
We also review how various proposals for safe AI are better understood as one class of methods, 
\emph{path-specific objectives}, which serve to remove the ICI.

We will introduce another new concept, the \emph{impact incentive} (II) (\cref{sec:ai-acp:ii}), which is more inclusive than the ICI.
Under some circumstances, an agent may be incentivized to influence some variable, not by its intention, 
but as a side-effect of optimal behaviour.
So an II will apply to any variables subject to an ICI, as well as those affected by a predictable side-effect.
IIs also have a sound and complete graphical criterion, that is a superset of the criterion for ICI.
We will also discuss how impact incentives can make sense of the 
purpose of \emph{impact measures} \citep{armstrong2017low,krakovna2018penalizing}, 
another proposal for safe AI.

We will then discuss various possible generalizations of incentive concepts to a multi-decision setting, 
and how they relate to one another (\cref{sec:ai-acp:multi-decision}).

Finally, we review related work (\cref{sec:ai-acp:related-work}), and conclude (\cref{sec:ai-acp:conclusion}).

This paper is an extended version of a conference paper, \citet{everitt2021agent}.
Since its publication, the concepts have already aided
understanding of incentive problems such as an agent's redirectability \citep{armstrong2020pitfalls,carey2023human}, 
ambition \citep{cohen2020unambitious}, fairness \citep{ashurst2022fair} tendency to tamper with reward 
\citep{everitt2021reward}, manipulativeness \citep{farquhar2022path}, the definition of an agent \citep{kenton2023discovering}, 
and more \citep{everitt2019modeling,langlois2021rl}.
Compared to that paper, \cref{sec:ai-acp:materiality,sec:ai-acp:response,sec:ai-acp:fci}
have been generalized to deal with multiple variables.
Analyses of intent and path-specific objectives have been newly added to \cref{sec:ai-acp:fci}.
Finally, \cref{sec:ai-acp:ii,sec:ai-acp:multi-decision} are entirely new.


\begin{figure*}
\begin{subfigure}[t]{\textwidth}
    \centering
      \begin{influence-diagram}
    \setrectangularnodes
    \setcompactsize
    \tikzset{node distance=3.5mm and 5.5mm}

\node (R) [] {Race};
      \node (S) [right= of R] {High school};
      \node (E) [right= of S] {Education};
      \node (Gr) [right=of E] {Grade};
\node (D) [below=of S,decision] {Predicted grade};
      \node (Ge) [left= of D] {Gender};

      \node (U) at (Gr|-D) [utility] {Accuracy};

    \draw[->]
      (R) edge (S)
      (S) edge (E)
      (S) edge (D)
      (E) edge (Gr)
      (Gr) edge (U)
      (Ge) edge (D)
      (D) edge (U)
    ;
\end{influence-diagram}
 \caption{Fairness example: grade prediction}\label{fig:race-preview}
  \end{subfigure}\vspace*{5mm}
  \begin{subfigure}[t]{\textwidth}
      \centering
      \hspace*{2cm}
      \begin{influence-diagram}
  \setrectangularnodes
  \setcompactsize
  \tikzset{node distance=4mm and 3.5mm}
  \node (D) [decision, anchor=west] {Posts\\ to show};
  \node (P1) [above left =3.5mm and -7mm of D] {Original\\ user opinions};
  \node (U) [right = of D, utility] {Clicks};
  \node (P2) at (U|-P1) {Influenced\\ user opinions};
\draw[->]
    (P1) edge (D)
    (P1) edge (P2)
    (D) edge (P2)
    (D) edge (U)
    (P2) edge (U)
  ;
\end{influence-diagram}
\hspace*{0.5cm}
    \begin{influence-diagram}
    \cidlegend[]{
      \legendrow{}{structure\! node} \\
      \legendrow{decision}{decision node}\\
      \legendrow{utilityc, chamfered rectangle xsep=1.5pt, chamfered rectangle ysep=1.5pt}{utility node}\\
}
\end{influence-diagram}
\hspace*{-2cm}
       \caption{Safety example: content recommendation}\label{fig:fci-preview}
  \end{subfigure}
  \caption[Grade predictor and content recommendation examples]{Two examples of decision problems represented as causal influence diagrams.
      In a) a predictor at a hypothetical university aims to estimate a student's grade, using as inputs
      their gender and the high school they attended.
      We ask whether the predictor is incentivized to behave in a discriminatory manner
      with respect to the students' gender and race.
      In this hypothetical cohort of students, performance is assumed to be a function of the quality of the high school education they received.
      A student's high school is assumed to be impacted by their race, and can affect the quality of their education.
      Gender, however, is assumed not to have an effect.
      In b) the goal of a content recommendation system is to choose posts that will maximize the user's click rate.
      However, the system's designers prefer the system not to manipulate the user's opinions in order
      to obtain more clicks.~\looseness=-1
      }
\label{fig:cid-preview}
\end{figure*}

\paragraph{Running examples}
For explanatory purposes, we will refer to the following pair of incentive design problems throughout the paper:

\begin{example}[Grade prediction]
To decide which applicants to admit, a university uses a model to predict the grades of new students.
The university would like the system to predict accurately, without treating students differently based on their gender or race
(see \cref{fig:race-preview}).~\looseness=-1 \end{example}

\begin{example}[Content recommendation]
An AI algorithm has the task of recommending a series of posts to a user.
The designers want the algorithm to present content adapted to each user's interests to optimize clicks.
However, they do not want the algorithm to use polarizing content to manipulate the user into clicking more predictably
(\cref{fig:fci-preview}).~\looseness=-1
\end{example}

\section{Setup}\label{sec:ai-acp:setup}
We will begin with a recap of
structural causal models and then introduce structural causal influence models.

\subsection{Structural causal models} \label{subsec:scm}

Structural causal models (SCMs) \citep{pearl2009causality}
are a type of causal model where all randomness is consigned to exogenous
variables, while deterministic structural functions relate the
endogenous variables to each other and to the exogenous ones.
As demonstrated by \citet{pearl2009causality}, this structural approach
has significant benefits over traditional causal Bayesian networks for
analysing (nested) counterfactuals and ``individual-level'' effects.

\begin{definition}[Structural causal model (unconfounded); {\citealp[Chapter 7]{pearl2009causality}}]\label{def:ai-acp:scm}
    A \emph{structural causal model} is a tuple
    $\langle \exovars, \evars, \structfns, P\rangle$, where $\exovars$ is a set of exogenous
    variables; $\evars$ is a set of endogenous variables;
    and $\structfns= \setfor{\fv{\evar}}{\evar \in \evars}$ is a collection of
    functions, one for each $\evar$.
    Each function $\fv{\evar}\colon \dom(\Pav{\evar} \cup \{\exovarv{\evar}\}) \to \dom({\evar})$
    specifies the value of $\evar$ in terms of
    the values of the corresponding exogenous variable $\exovarv{\evar}$
    and a set of variables $\Pav{\evar} \subset \evars$, where
    these functional dependencies are acyclic.%
    \footnote{The reason for using the notation $\Pav{\evar}$ to designate this set of variables will become clear 
    when we introduce the ``associated DAG'' later in this subsection.}
    The domain of a variable $\evar$ is $\dom(\evar)$,
    and for a set of variables, ${\dom(\sW) := \bigtimes_{W \in \sW}{\dom(W)}}$.
The uncertainty is encoded through a probability distribution $P(\exovals)$ such that the exogenous variables are mutually independent.~\looseness=-1
\end{definition}

For example, \cref{fig:counterfactual1} shows an SCM that models how
\emph{posts} ($D$) can influence a user's \emph{opinion} ($O$) and \emph{clicks}
($U$).

The exogenous variables $\exovars$ of an SCM represent factors that are not
modelled.
For any value $\exovars = \exovals$ of the exogenous variables, the value of
any set of variables $\sW \subseteq \evars$ is given by recursive application of
the structural functions $\structfns$ and is denoted by $\sW(\exovals)$.
Together with the distribution $\exoprob(\exovals)$ over exogenous variables, this induces a
joint distribution ${\Pr(\sW = \sw)}
= \sum_{\{\exovals|\sW(\exovals)=\sw\}}{\exoprob(\exovals)}$.~\looseness=-1

Note that in general, we denote individual variables by capital letters, 
and sets of variables by bolded capital letters.
Individual (sets of) assignments will be represented by 
(bolded) lowercase. \label{pg:irici:variables}

SCMs model \emph{causal interventions} that set variables to particular values.
These are defined via submodels:

\begin{definition}[Submodel; {\citealp[Chapter 7]{pearl2009causality}}]\label{def:ai-acp:submodel}
Let $\scm = \scmdef$ be an SCM, $\sX$ a set of variables in $\evars$, and
$\sx$ a particular realization of $\sX$.
The submodel $\scm_\sx$ represents the effects of an \emph{intervention}
$\Do(\sX=\sx)$,
and is formally defined as the SCM
${\langle \exovars, \evars, \structfns_\sx, \exoprob \rangle}$,
where
    ${\structfns_\sx = \{\fv{\evar} | \evar \notin \sX\} \cup
    {\{\sX = \sx\}}}$.
That is to say, the original functional relationships of $X \in \sX$ are
replaced with the constant functions $X = x$.
\end{definition}

More generally, a \emph{soft intervention} on a variable $X$ \label{pg:irici:soft-intervention}
in an SCM $\scm$ replaces $\fv{X}$ with a
function $\gx\colon \dom(\Pav{X} \cup \{\exovarv{X}\}) \to \dom(X)$
\citep{eberhardt2007interventions,tian2013causal}.
The probability distribution $\Pr(\sW_{\gx})$
on any $\sW \subseteq \evars$
is defined as the value of $\Pr(\sW)$ in the submodel $\scim_{\gx}$
where $\scim_{\gx}$ is $\scim$
modified by replacing $\fv{X}$ with $\gx$.

If $W$ is a variable in an SCM $\scim$, then $W_{\sx}$ refers to the same variable
in the submodel $\scim_{\sx}$, and is called a \emph{potential response variable}.
In \cref{fig:counterfactual1}, the random variable $O$ represents user opinion
under ``default'' circumstances,
while $O_d$ in \cref{fig:counterfactual3} represents the
user's opinion given an intervention $\doo(D=d)$ on the content posted. 
Note also how the intervention on $D$ severs the link from $\exovarv{D}$ to $d$
in \cref{fig:counterfactual3}, as the intervention on $D$ overrides the causal
effect from $D$'s parents.
Throughout this paper we use subscripts to indicate submodels or interventions,
and superscripts for indexing. \label{pg:irici:doo}

\begin{figure*}
  \centering
\newlength{\mysubfigwidth}
  \pgfmathsetlength{\mysubfigwidth}{40mm}
  \captionsetup[subfigure]{oneside,margin={0mm,0mm}}
  \begin{subfigure}[t]{0.3\textwidth}
    \centering
    \begin{influence-diagram}
      \setcompactsize
      \setinnersep{0.5mm}
      \node (D) [decision] {$D$};
      \node (X) [below =7mm of D] {$O$};
      \node (Y) [below =7mm of X,utility] {$U$};
        \node (ed) [left = 3mm of D, exogenous] {$\exovarv{D}$};
        \node (ex) [left = 3mm of X, exogenous] {$\exovarv{O}$};
        \node (ey) [left = 2.5mm of Y, exogenous] {$\exovarv{U}$};
      \path
      (ed) edge[->] (D)
      (ex) edge[->] (X)
      (ey) edge[->] (Y)
      (D) edge[->] (X)
      (X) edge[->] (Y)
      ;
      \draw[->] (D) to[in=40,out=-40,looseness=.6] (Y);
      \begin{scope}[
        every node/.style={draw=none,rectangle,align=center,inner sep=0mm}
        ]
        \scriptsize
\node[right = 3mm of X] {Opinion \\ \textcolor{dmblue500}{$O = f_O(D, \exovarv{O})$}};
        \node[right = 0mm of Y] {Clicks\\ \textcolor{dmblue500}{$U \!=\! f_U(\! D, O,\exovarv{U}\!)$}};
      \end{scope}
      \node (h2) [below = 4mm of Y, inner sep=0mm, minimum size=0mm] {};
    \end{influence-diagram}
    \caption{SCIM}\label{fig:scim-example}
\end{subfigure}
  \begin{subfigure}[t]{0.3\textwidth}
    \centering
    \begin{influence-diagram}
      \setcompactsize
      \setinnersep{0.5mm}
      \node (D) [] {$D$};
      \node (X) [below =7mm of D] {$O$};
      \node (Y) [below =7mm of X] {$U$};
        \node (ed) [left = 3mm of D, exogenous] {$\exovarv{D}$};
        \node (ex) [left = 3mm of X, exogenous] {$\exovarv{O}$};
        \node (ey) [left = 3mm of Y, exogenous] {$\exovarv{U}$};
      \path
      (ed) edge[->] (D)
      (ex) edge[->] (X)
      (ey) edge[->] (Y)
      (D) edge[->] (X)
      (X) edge[->] (Y)
      ;
      \draw[->] (D) to[in=40,out=-40,looseness=.6] (Y);
      \begin{scope}[
        every node/.style={draw=none,rectangle,align=center,inner sep=0mm}
        ]
        \scriptsize
        \node[right = 2mm of D] {Posts \\ \textcolor{dmblue500}{$D = \pi(\exovarv{D})$}};
        \node[right = 3mm of X] {Opinion \\ \textcolor{dmblue500}{$O = f_O(D, \exovarv{O})$}};
        \node[right = 0mm of Y] {Clicks\\ \textcolor{dmblue500}{$U \!=\! f_U(\!D, O,\exovarv{U}\!)$}};
      \end{scope}
      \node (h2) [below = 4mm of Y, inner sep=0mm, minimum size=0mm] {};
    \end{influence-diagram}
    \caption{SCM}\label{fig:counterfactual1}
\end{subfigure}\hspace{-1mm}
  \begin{subfigure}[t]{0.39\textwidth}
    \centering
    \begin{influence-diagram}
      \setcompactsize
      \setinnersep{0.5mm}
      \node (D) [] {$D$};
      \node (X) [below =7mm of D] {$O$};
      \node (Y) [below =7mm of X] {$U$};
      \node (ed) [left =2.5mm of D, exogenous] {$\exovarv{D}$};
      \node (ex) [left =2.5mm of X, exogenous] {$\exovarv{O}$};
      \node (ey) [left =2.5mm of Y, exogenous] {$\exovarv{U}$};
      \node (d) [double,left = of ed] {$d$};
      \node (Xd) [left = of ex] {$O_d$};
      \node (YXd) [left = of ey] {$U_{O_d}$};
      \path
      (ed) edge[->] (D)
      (ex) edge[->] (X)
      (ex) edge[->] (Xd)
      (ey) edge[->] (Y)
      (ey) edge[->] (YXd)
      (D) edge[->] (X)
      (d) edge[->] (Xd)
      (X) edge[->] (Y)
      (Xd) edge[->] (YXd)
      (D) edge[->,bend left] (Y)
      ;
      \node (h1) [right = 3mm of Y, inner sep=0mm, minimum size=0mm] {};
      \node (h2) [below = 3mm of Y, inner sep=0mm, minimum size=0mm] {};
      \path (D) edge[out=-55,in=90] (h1)
      (h1) edge[out =-90, in=0] (h2)
      (h2) edge[->, out =180, in=-15] (YXd)
      ;
      \begin{scope}[node distance=0.1cm and 0.1cm,
        every node/.style={ draw=none, rectangle,align=center,
          inner sep=0mm }
        ]
        \scriptsize
        \node [left = of d] {Posts \\ \textcolor{dmblue500}{$d=\text{apolitical}$ }};
        \node [left = of Xd] { Opinion\\
          \textcolor{dmblue500}{$O_d = f_O(d, \exovarv{O})$}};
          \node [left =0mm of YXd] {Clicks\\
          \textcolor{dmblue500}{$U_{O_d} \!=\!$} \\\textcolor{dmblue500}{$f_U(\!D,O_d,\exovarv{U}\!)$}};
        \end{scope}

    \end{influence-diagram}
\vspace{0mm} 
\caption[Nested counterfactual in twin graph]{SCM with nested counterfactual, shown using a twin graph \citep{balke2022probabilistic}}\label{fig:counterfactual3}
\end{subfigure}
\begin{influence-diagram}
    \cidlegend[]{
  \node () [draw, circle, exogenous] {}; \&
  \node [draw=none, rectangle] {exogenous node}; \&
  \node () [draw, circle] {}; \&
  \node [draw=none, rectangle] {structural node}; \&
  \node () [draw, circle, double] {}; \&
  \node [draw=none, rectangle] {intervened node}; \\
  \node () [draw, circle, decision] {}; \&
  \node [draw=none, rectangle] {decision node}; \&
  \node () [draw, circle, utilityc, chamfered rectangle xsep=1.5pt, chamfered rectangle ysep=1.5pt] {}; \&
  \node [draw=none, rectangle] {utility node}; \\
}

\end{influence-diagram}
  \caption[Interventions in a structural causal influence model]{
    An example of a SCIM and interventions.
    In the SCIM, either political or apolitical posts $D$ are displayed.
    These affect the user's opinion $O$. $D$ and $O$ influence the user's clicks $U$ (a).
    Given a policy, the SCIM becomes an SCM (b).
    Interventions and counterfactuals may be defined in terms of this SCM.
    For example, the nested counterfactual $U_{O_d}$ represents
    the number of clicks if the user has the opinions that they would
    arrive at, after viewing apolitical content (c).}
  \label{fig:counterfactual}

\end{figure*}

More elaborate hypotheticals can be described with
a nested counterfactual. \label{pg:irici:nested}
In a nested counterfactual, the intervention is itself a potential response 
variable.
For instance, in \cref{fig:counterfactual3}, 
we may be interested in what the utility would be if the user's \emph{opinions} 
assumed the value that they would take given some alternative \emph{posts}.
Put differently, we would like to propagate the effect of an intervention $\doo(D=d)$ 
to $U$, only via the opinions $O$.
To define a nested counterfactual, firstly, 
the value $o=O_d(\exovals)$ indicates the user's opinion after receiving a 
default post $D=d$, 
given an assignment $\exovals$ to the exogenous variables. 
Then, the effect of the intervention $\doo(O=o)$ on the user's \emph{clicks} $U_{O_d}$
is defined as $U_{O_d}(\exovals) \coloneqq U_{o}(\exovals)$ for any assignment $\exovals$.

A structural causal model has an \emph{associated DAG}
that can be used to deduce which variables are 
conditionally independent.
Formally, the induced graph has vertices $\evars$ and
an edge inbound to each variable $\evar$ from each 
variable that $f_\evar$ depends on.
For example, in \cref{fig:counterfactual1}, the dependencies 
of the functions $\pi,f_O,f_U$ are illustrated.
In \cref{def:ai-acp:scm}, we designated the variables that $\evars$ 
depends on as $\Pav{\evar}$,
and this is because they are the parents of $\evar$ in the associated DAG.
In fact, for any DAG, we will use the same notation $\Pav{\evar}$ to designate the parents of a variable 
$\evar$, and similarly $\Descv{\evar}$ to designate the descendants. \label{pp:irici:padesc}
We will use some more standard notation for DAGs: 
an edge from node $V$ to node $Y$ is denoted $V \to Y$, and
a directed path (of length at least zero) is denoted $V \pathto Y$.

The d-separation criterion can be used to deduce when two sets
of variables are independent, conditional on another variable.

\begin{definition}[d-separation; \citealp{Verma1988soundness}]\label{def:ai-acp:d-separation}
    A path $p$ is said to be d-separated by a set of nodes $\sZ$ if and only if:
    \begin{enumerate}
    \item $p$ contains a collider $X \to W \gets Y$ such that the middle node $W$ is not in $\sZ$ and no descendants of $W$ are in $\sZ$, or
    \item $p$ contains a chain $X \to W \to Y$ or fork $X \gets W \to Y$ where $W$ is in $\sZ$, or
    \item one or both of the endpoints of $p$ is in $\sZ$.
    \end{enumerate}
    A set $\sZ$ is said to d-separate $\sX$ from $\sY$,
written ${(\sX \perp \sY \mid \sZ)}$, if and only if $\sZ$ d-separates every path
from a node in $\sX$ to a node in $\sY$. Sets that are not d-separated are
called d-connected.~\looseness=-1
\end{definition}

When d-separation holds,
these sets of variables must be independent given the third.
Conversely, when variables are d-connected in a graph, 
then there exists a model with that induced graph 
such that they are conditionally dependent.

\begin{theorem}[Theorem 1.2.4 of \citet{pearl2009causality}]
If sets $\sX,\sY,\sZ$ satisfy $\sX \perp \sY \mid \sZ$ in a DAG $\calG$, 
then $\sX$ is independent of $\sY$ conditional on $\sZ$ in every 
SCM $\calM$ with induced graph $\calG$.
Conversely, if $\sX \not\perp \sY \mid \sZ$ in a DAG 
$\calG$, then $\sX$ and $\sY$ are dependent conditional on $\sZ$ in at least one SCM $\calM$ with induced graph $\calG$.
\end{theorem} 

Indeed, when variables are d-connected, they are actually 
conditionally dependent in almost all models with that induced 
graph \citep{meek1995strong}.

\subsection{Structural causal influence models}
Influence diagrams are graphical models 
with special decision and utility nodes,
used to model decision-making problems \citep{howard1990influence,lauritzen2001representing}, 
but that usually do not deal with counterfactual 
concepts as do SCMs \citep{heckerman1994decision}.
So for our analysis, we introduce a hybrid of SCMs and influence diagrams
called the \emph{structural causal influence model} 
(SCIM, pronounced ``skim'').
This model, originally proposed by \citet{dawid2002influence}, 
is essentially an SCM where particular variables 
are designated as decisions and utilities.
The decisions lack structural functions, until one is 
selected by an agent.\footnote{Dawid called this a ``functional influence diagram''.
  We favour the term SCIM, because the term ``SCM''
  is more prevalent than the corresponding term ``functional model''.
}

\begin{definition}[Structural causal influence model]\label{def:ai-acp:scim}

    A \emph{structural causal influence model} (SCIM)
is a tuple $\scim = \scimdef$ where:
\begin{itemize}
  \item $\langle \exovars, \evars, \structfns', \exoprob \rangle$ is an unconfounded SCM, and $\structfns = \structfns' \setminus \structfns_{\decisionvars}$ consists of the structural functions 
  from that SCM, except those belonging to a set $\decisionvars \subseteq \evars$, called \emph{decision variables}.
  \item The \emph{utility variables} $\utilvars$ are a  
  subset of $\evars \setminus \decisionvars$, and have real domains, $\dom(U) \subseteq \mathbb{R}$ for all $U \in \utilvars$. By convention, we 
  require that utility variables have no children in the associated DAG.
  \item The observation function $\mathcal{O}$ maps each decision 
  variable $\decisionvar \in \decisionvars$ to a set of \emph{observed variables} $O \subseteq \evars \setminus \utilvars$, intuitively, the variables that $D$ can depend on.
\end{itemize}
\end{definition}
Those endogenous variables that are neither decisions 
nor utilities are called \emph{structural variables}, $\structvars := \evars \setminus (\decisionvars \cup \utilvars)$

A SCIM entails an acyclic relationship between all of its variables, which can be represented by a DAG.
The observation function $\mathcal{O}$ 
indicates which variables are available as inputs to any given 
decision variable --- these will be the parents.
For non-decision variables, the parents are implied 
by the structural functions $\structfns$, which indicate
the variable's direct causes.

\footnote{
In the study of structural causal models, 
the variables that are not exogenous are often called ``visible'' 
and a joint distribution over visible variables is available to the decision-maker.
In a SCIM, the decision-maker instead has access to 
the SCIM tuple, along with assignments to observations.
We therefore avoid referring to any nodes as ``visible''.}
Taken together, these allow us to associate any SCIM 
with an influence diagram --- a DAG that illustrates these 
dependencies, as well as the types of each variable.

\begin{definition}[Causal influence diagram]\label{def:ai-acp:cid}
The \emph{causal influence diagram} (CID) of a SCIM
is a graph
whose vertices are the decision nodes $\decisionvars$, 
structure nodes $\structvars$, and utility nodes $\utilvars$,
and whose edges go
from observations $\mathcal{O}(D)$ to each decision $D$
and from variables that $f^V$ depends on, to each non-decision $V$.
\end{definition}

The arcs into each decision are ``informational'' 
in that they indicate which parents of the decision will 
be observed by the decision maker at the time that decision is selected \citep{shachter2016decisions}.
We will focus exclusively on SCIMs whose CID is acyclic.

An example of a SCIM for the content recommendation example
is shown in \cref{fig:scim-example}, and the node types 
of the CID are highlighted in a standard style --- 
the decision nodes as rectangles, and the utilities as diamonds.

In single-decision SCIMs,
the decision-making task is to maximize expected utility by selecting
a decision $d\in \dom(D)$ for each possible assignment to the observations $o \in \dom(\mathcal{O}(D))$, i.e.\ 
to select a \emph{decision rule}
$\pi^D: \dom(\mathcal{O}(D) \cup \{\exovarv{\decisionvar}\}) \to \dom(\decisionvar)$.
The exogenous variable $\exovarv{D}$ provides randomness to allow
the decision rule to be a stochastic function of the 
observations $\mathcal{O}(D)$.%
\footnote{
Ideally, we might want the decision-maker to be able to implement 
\emph{any} stochastic policy.
This could be done by having $\exovarv{\decisionvar}$ be a continuous random variable. 
However, this would introduce measure theoretic complications that 
are not pertinent to the analysis in this paper, and so we 
defer that construction to future work.
}
If there are multiple decisions, the task is to select
a \emph{policy} $\spi = \{\pi^D\}_{D \in \sD}$, i.e.\ 
one decision rule for each decision. \label{pg:irici:policy}
Specifying a policy turns a SCIM $\scim$ into an SCM
$\scim_\spi := \langle \exovars, \evars, (\structfns \setminus \structfns_{\decisionvars}) \cup \spi, \exoprob
\rangle$.
In the resulting SCM, the standard definitions of causal interventions apply.

We use $\Prs{\spi}$ and $\EEs{\spi}$ to denote probabilities and expectations
with respect to $\scims{\spi}$. \label{pg:irici:ppi}
For a set of variables $\sX$ not in $\Descv{\decisionvar}$,
$\Prs{\spi}(\sx)$ is independent of $\spi$ and we simply write $\Pr(\sx)$.
An \emph{optimal policy} for a SCIM is defined as any policy $\spi$
that maximizes $\EEs{\spi}[\totutilvar]$, where
$\totutilvar \coloneqq \sum_{\utilvar \in \utilvars}{\utilvar}$.
The potential response $\totutilvar_\sx$ is defined as \label{pg:irici:total-utility}
$\totutilvar_\sx \coloneqq \sum_{\utilvar \in \utilvars}{\utilvar_\sx}$.~\looseness=-1 
In most of the examples that we consider, there will only be one decision, 
and so by slight abuse of notation, we will 
denote the policy $\spi = \{\pi\}$ by $\pi$.

Finally, let us clarify why a CID is called ``causal''.
For an ordinary influence diagram, 
one can deduce that only the descendants of a decision 
are caused by it, because their values depend on the chosen policy \citep{heckerman1994decision}.
In a CID, however, imputing a policy 
recovers a structural causal model, which 
represents a full description of causal relationships between variables.
The direction of causality then corresponds to the direction of 
arrows in the associated DAG. 
Since these arrows are the same as those in the original CID, 
we may also call the CID \emph{causal}.

\section{Materiality} \label{sec:ai-acp:materiality}
A fundamental question that we may ask about the optimal policies
is: which observations do they need in order to make optimal decisions?
If some observation is discovered to be \emph{immaterial} \citep{shachter2016decisions}, 
this would allow us to narrow 
the search for optimal policies.
Conversely, if an observation is \emph{material}, this means it will directly influence the decision under every
optimal policy.\footnote{In contrast to subsequent sections, the results in this section and
  the VoI section do not require the influence diagrams to be causal.}~\looseness=-1

\begin{definition}[Materiality; \citealp{shachter2016decisions}] \label{def:irici:materiality}
  For any given SCIM $\scim$, let
  $\attutil(\scim)=\max_{\pi}\EEs{\pi}[\totutilvar]$ be the maximum attainable
  utility in $\scim$, and let $\scim_{\incentivevar \not \to D}$ be the modified version of $\scim$ 
  obtained by removing the information links from $\incentivevar$ to $D$.
  The observation $\incentivevar\subseteq \Pad$ is \emph{material} if
  $\attutil(\scim_{\incentivevar \not \to D})
  <
  \attutil(\scim)
  $.
\end{definition}
 
Nodes may often be identified as immaterial
based on the graphical structure alone \citep{fagiuoli1998note,lauritzen2001representing,shachter2016decisions}.
According to the graphical criterion of \citet{fagiuoli1998note}, an observation
cannot provide useful information if it is d-separated from utility, conditional on
other observations.
This condition is called \emph{non-requisiteness}.

\begin{definition}[Non-requisite observation; \citealp{lauritzen2001representing}]
  \label{def:ai-acp:requisite-observation}
  Let $\utilvarsd := \utilvars\cap\Descv{\decisionvar}$ be
  the utility nodes downstream of $\decisionvar$.
  An observation $\incentivevar\in \Pad$ in a single-decision CID $\causalgraph$ is
  \emph{non-requisite} if:  \begin{equation}
    \label{eq:voi-criterion}
    \incentivevar \perp \utilvars^D \bmid \left(\Pad \cup \{\decisionvar\} \setminus \{\incentivevar\}\right).
  \end{equation}
  In this case, the edge $\incentivevar\to \decisionvar$ is also called non-requisite.
  Otherwise $\incentivevar$ and $\incentivevar\to \decisionvar$ are \emph{requisite}.
\end{definition}
 
Variables that are non-requisite are immaterial.

\begin{restatable}[Materiality criterion]{theorem}{MaterialityCriterion}
  \label{th:materiality}
  A single decision CID $\causalgraph$ is compatible 
  with 
  $\incentivevar \in \evars$
  being material
    if and only if $\incentivevar$ is a requisite observation in $\causalgraph$.
\end{restatable}

The proof is supplied in appendix ~\ref{app:ai-acp:appendix-ri}.
The soundness direction (i.e.\ the \emph{only if} direction) is well-known, 
and follows from d-separation \citep{fagiuoli1998note,lauritzen2001representing,shachter2016decisions}.
In contrast, the completeness direction does not follow from the
completeness property of d-separation.
The d-connectedness of $\incentivevars$ to $\utilvars$ implies that $\utilvars$ may be conditionally dependent on $\incentivevars$.
It does not imply, however, that the expectation of $\utilvars$ or the utility attainable under an optimal policy will change.
Instead, our proof constructs a SCIM where some $\incentivevar \in \incentivevars$ is material.
This differs 
from a previous attempt by \citet{nielsen1999welldefined}
that is reviewed in \cref{sec:ai-acp:related-work}.

Let us now apply the graphical criterion to the grade prediction example in \cref{fig:race-a}.
Here, \emph{gender} is a non-requisite observation.
This means that gender is conditionally independent of grade 
given the high school and predicted grade.
So it can provide no useful information for predicting the university 
grade, given what else the predictor knows.
On the other hand, high school is a requisite observation,
so it may be required to make an optimal prediction.

Materiality asks whether a variable that is observed is necessary for optimal performance.
We can generalize this to unobserved variables, by also asking whether performance 
would be improved by observing an additional variable.
This concept, \emph{value of information}, is treated in \cref{app:ai-acp:voi}.

 \section{Response incentives}\label{sec:ai-acp:response}
One way to understand materiality is that a material observation 
is one that influences optimal decisions.
So, a natural generalization is the set of all (observed and latent) variables that influence the decision. We say that these variables have a response incentive.\footnote{The term \emph{responsiveness} \citep{heckerman1995decision,shachter2016decisions}
  has a related but not identical meaning --
it refers to whether a decision $D$ affects a variable $\incentivevar$
rather than whether $\incentivevar$ affects $D$.}

\begin{definition}[Response incentive]\label{def:ai-acp:response-incentive}
Let $\scim$ be a single-decision SCIM.
A policy $\pi$ \emph{responds} to variables $\incentivevars\subseteq \structvars$ if
there exists some set $g^\incentivevars$ of soft interventions, 
one $g^\incentivevar$ for each $\incentivevar \in \incentivevars$,
and some setting $\exovars = \exovals$, such that
$\decisionvar_{g^\incentivevars}(\exovals) \ne \decisionvar(\exovals)$.
The variables
$\incentivevars$ have a \emph{response incentive} if all optimal
policies respond to $\incentivevars$.

\end{definition}

For a response incentive on $\incentivevars$ to be possible, there must be:
i) a directed path $\incentivevar \pathto D$ for some $\incentivevar \in \incentivevars$, and 
ii) an incentive for $D$ to use information from that path.
For example, in \cref{fig:race-a}, \emph{gender} has a directed path to the decision
but it does not provide any information about the likely grade, so there is no response incentive.
The graphical criterion for RI builds on a modified graph with non-requisite
information links removed.

\begin{definition}[Minimal reduction; \citealp{lauritzen2001representing}] \label{def:irici:minimal-reduction}
  \label{def:ai-acp:reduced-graph}
  The \emph{minimal reduction} $\reducedgraph$ of a single-decision CID $\causalgraph$
  is the result of removing from $\causalgraph$ all information links from non-requisite observations.
\end{definition}

The presence (or absence) of a path $\incentivevar \pathto D$ in the minimal reduction tells us whether
a response incentive can occur.~\looseness=-1

\begin{restatable}[Response incentive criterion]{theorem}{ResponseIncentiveCriterion}\label{theorem:ri-graph-criterion}
A single‑decision CID \(\cid\) admits a response incentive on \(\incentivevars \subseteq \structvars\) if and only if the minimal reduction \(\reducedgraph\) has a directed path \(\incentivevar \pathto \decisionvar\) for some \(\incentivevar \in \incentivevars\).
\end{restatable}

The intuition behind the proof is that an optimal decision only responds to
effects that propagate to one of its requisite observations.
For the completeness direction, we show in \cref{app:ai-acp:appendix-ri}
that if $\incentivevar \pathto D$ is present in the minimal reduction $\reducedgraph$, then
we can select a SCIM $\scim$ compatible with $\cid$ such that $D$ receives useful
information along that path, that any optimal policy must respond to.

In a setting where an agent has an option to shut down,
safe behaviour requires a condition called 
\emph{obedience}, which requires the system to respond to any 
shutdown instruction that is given \citep{carey2023human}.
For algorithms designed for human assistance, incentivising 
responsiveness in this way has been an important desideratum \citep{Hadfield-Menell2016osg}.

In a fairness setting, on the other hand, a response incentive may be a cause
for concern, as illustrated next.~\looseness=-1

\begin{figure}
\begin{subfigure}[t]{\linewidth}
    \centering
      \begin{influence-diagram}
    \setrectangularnodes
    \setcompactsize
    \tikzset{node distance=5.5mm and 5.5mm}

\node (R) [] {Race};
      \node (S) [right= of R] {High school};
      \node (E) [right= of S] {Education};
      \node (Gr) [right=of E] {Grade};
\node (D) [below=of S,decision] {Predicted grade};
      \node (Ge) [left= of D] {Gender};
      
      \node (U) at (Gr|-D) [utility] {Accuracy};

    \draw[->]
      (R) edge (S)
      (S) edge (E)
      (S) edge (D)
      (E) edge (Gr)
      (Gr) edge (U)
      (Ge) edge (D)
      (D) edge (U)
    ;

\voiincentive{S}
\riincentive{R}
    \riincentive[inner sep=2.2mm]{S}

\end{influence-diagram}

 \caption{
        Admits response incentive on race
    }\label{fig:race-a}
\end{subfigure}
  \begin{subfigure}[t]{\linewidth}
    \centering
      \begin{influence-diagram}
    \setrectangularnodes
    \setcompactsize
    \tikzset{node distance=3.5mm and 4.5mm}

    \node (R) [] {Race};
      \node (S) [right= of R] {High\\ school};
      \node (E) [right= of S] {Education};
      \node (Gr) [right=of E] {Grade};
\node (D) [below=of S, xshift=4mm, decision] {Predicted grade};
      \node (Ge) [left= of D] {Gender};
      
      \node (U) at (Gr|-D) [utility] {Accuracy};

    \draw[->]
      (R) edge (S)
      (S) edge (E)
(E) edge (Gr)
      (Gr) edge (U)
      (Ge) edge (D)
      (D) edge (U)
;
  \end{influence-diagram}
    \hspace*{0.5cm}
  \begin{influence-diagram}
    \cidlegend[right = 5mm of Gr, yshift=-2mm]{
      \legendrow{value of information}{Material} \\
      \legendrow{response incentive}{Response\\ incentive} \\
}
  \end{influence-diagram}
\hspace*{-2.5cm}
     \caption{
      Admits no response incentive on race
    }\label{fig:race-b}
  \end{subfigure}
  \caption[Incentives for the grade-predictor]{In (a), the admissible incentives of the grade prediction example from \cref{fig:race-preview} are shown,
      including a response incentive on race.
      In (b), the predictor no longer has access to the students' high school,
      and hence there can no longer be any response incentive on race.~\looseness=-1
  }\label{fig:race}
\end{figure}
 \paragraph{Incentivised unfairness}
Response incentives are closely related to counterfactual fairness \citep{kusner2017counterfactual,kilbertus2017avoiding}.
A prediction --- or more generally a decision --- is considered counterfactually
unfair if a change to a \emph{sensitive attribute} like race or gender would
change the decision.

\begin{definition}[Counterfactual fairness; {\citealp{kusner2017counterfactual}}]
A policy $\pi$ is \emph{counterfactually fair} with respect to a
  sensitive attribute $A$ if
\begin{equation*}
    \label{eq:cf}
    \Prs{\pi}\left(
      \decisionvar_{a'} = \decisionval\mid \pad, a \right) =
    \Prs{\pi}\left(
      \decisionvar = \decisionval \mid \pad, a\right)
  \end{equation*}
  for every decision $\decisionval \in \dom(\decisionvar)$,
  every context $\pad \in \dom(\Pad)$, and
  every pair of attributes $a, a' \in \dom(A)$ with $\Pr(\pad, a) > 0$.
\end{definition}
 
A response incentive on a sensitive attribute indicates that counterfactual unfairness is incentivized,
as it implies that \emph{all} optimal policies are counterfactually unfair:

\begin{restatable}[Counterfactual fairness and response incentives]{theorem}{theoremcffair}\label{theorem:counterfactual-fairness}
In a single-decision SCIM $\scim$ with a sensitive attribute $A\in\rivars$,
all optimal policies $\pi^*$ are counterfactually unfair
with respect to $A$ if and
only if $\{A\}$ has a response incentive.
\end{restatable}

The proof is given in \cref{app:ai-acp:appendix-fairness}.

A response incentive on a sensitive attribute means
that counterfactual unfairness is not just possible, but incentivized.
As a result, the graphical criterion for a response incentive 
is more restrictive than the graphical criterion for 
counterfactual unfairness being possible.
The latter requires only that a sensitive attribute be an ancestor of the decision \citep[Lemma 1]{kusner2017counterfactual}.
For example, in the grade prediction example of \cref{fig:race-a},
it is possible for a predictor to be counterfactually unfair
with respect to either \emph{gender} or \emph{race}, because both are ancestors of the decision.
The response incentive criterion can tell us whether counterfactual unfairness may actually be incentivized.
In this example, the minimal reduction includes the edge from
\emph{high school} to \emph{predicted grade} and hence the directed path from \emph{race} to \emph{predicted grade}.
However, it excludes the edge from \emph{gender} to \emph{predicted grade}.
This means that the agent is incentivized to be counterfactually unfair with respect to
\emph{race} but not to \emph{gender}.~\looseness=-1

Based on this, how should the system be redesigned? According to the response incentive criterion,
the most important change is to remove the path from
\emph{race} to \emph{predicted grade} in the minimal reduction. This
can be done by removing the agent's access to \emph{high school}.
This change is implemented in \cref{fig:race-b}, where there is no
response incentive on either
sensitive variable.~\looseness=-1


The incentive approach is not restricted to counterfactual fairness.
For any fairness definition, one could assess whether that kind of unfairness is incentivized
by checking whether it is present under all optimal policies.
For example, \citet{ashurst2022fair} considers whether unfairness is introduced
--- in that the prediction has greater demographic disparity than the true label ---
and establishes when this is incentivized.

\section{Instrumental control incentives}\label{sec:ai-acp:fci}
Let us return to the second running example, shown in \cref{fig:fci-preview}, 
where developers seek to anticipate harmful consequences of 
deploying a content recommender system.
A key concern they will have is that the system is incentivized 
to manipulate users' preferences.
In general, to describe whether an agent has to strategically 
influence some variable, we will define a notion of an 
\emph{instrumental control incentive}.
(This will also correspond to the notion of `convergent 
instrumental goals' described in the introduction.)
Note that this differs from the notion of value of control \citep{shachter1986evaluating}, which only considers the agent's need to 
influence a variable, and not its ability.
Value of control and its graphical criterion are analysed in \cref{app:ai-acp:voc}.


To formalize this question, we can consider whether an agent's  
influence on a variable $\incentivevar$ 
affects the policy's performance.
The effect of an alternative decision $d$ on the variable 
$\incentivevar$ can be written as $\incentivevar_d$.
And the effect of an alternative value $\incentiveval$ 
on the outcome $\utilvars$ can be written as 
$\utilvars_\incentiveval$.
Putting these together, the effect of setting $\incentivevar$ 
to the value obtained under $d$ is denoted by the nested 
counterfactual $\totutilvar_{\incentivevar_d}$, 
as defined in \cref{subsec:scm}.
If the performance of optimal policies is sensitive to such 
an intervention, then we will say there is an 
instrumental control incentive.

\begin{definition}[Instrumental control incentive]\label{def:ai-acp:instrumental-goal}
In a single-decision SCIM $\scim$,
there is an \emph{instrumental control incentive} on nodes $\incentivevars$ in decision context
$\pad$ if, for all optimal policies $\pi^*$,
there exists an alternative assignment $D=d$ such that:
\begin{equation}\label{eq:ci}
\EEs{\pi^*}{[\totutilvar_{\incentivevars_{\decisionval}} \mid \pad]}
\neq
\EEs{\pi^*}{[\totutilvar \mid \pad]}.
\end{equation}
\end{definition}

ICIs only consider the influence of $\incentivevar$ 
that is instrumental to achieving utility --- 
in the terminology of \citet{pearl2001direct}, a
\emph{natural indirect effect} from
$D$ to $U$ via $\incentivevar$ in $\scim_{\pi^*}$, for all optimal policies $\pi^*$.
ICIs do not consider side-effects shared by optimal policies:
for instance, it may be that all optimal policies affect $\incentivevar$ in a particular
way, even if $\incentivevar$ is a not an ancestor of any utility node, and
in such cases, no ICI is present.

\begin{restatable}[Instrumental Control Incentive Criterion]{theorem}{InstrumentalControlIncentiveCriterion}\label{theorem:ici-graph-criterion}
A single-decision CID $\cid$ admits an instrumental control 
incentive on
$\incentivevars \subseteq \civars$ if and only if $\cid$ has a directed
path from the decision $\decisionvar$ to a utility node $\utilvar \in \utilvars$
that passes through some $\incentivevar \in \incentivevars$.
\end{restatable}

The logic behind the soundness proof is that if there is no path from $D$ to some $\incentivevar \in \incentivevars$ to
$\utilvars$, then $D$ cannot have any effect on $\utilvars$ via $\incentivevars$.
For the completeness direction,
we show how to construct a SCIM
so that $U_{\incentivevars_d}$ differs from the non-intervened $U$ for any diagram with a
path $D\pathto \incentivevar \pathto \utilvars$
for any $\incentivevar \in \incentivevars$.

Let us apply this criterion to the content recommendation example
in \cref{fig:fci-application1}.
The only nodes $\incentivevar \in \incentivevars$ in this graph that lie on a path $D \pathto \incentivevar \pathto \utilvar$
for any $\utilvar \in \utilvars$
are \emph{clicks} and \emph{influenced user opinions}.
Since \emph{influenced user opinions} has an instrumental control incentive,
the agent may seek to influence that variable in order to attain utility.
For example, it may be easier to predict what content a more emotional user will
click on and therefore, a recommender may
achieve a higher click rate by introducing posts that
induce strong emotions.~\looseness=-1

\newcommand{\casubfigwidth}{0.475\textwidth}

\begin{figure}
\begin{subfigure}[t]{0.475\textwidth}
    \centering
    \begin{influence-diagram}
  \setrectangularnodes
  \setcompactsize

  \node (D) [decision] {Posts\\ to show};
  \node (P1) [above left =6mm and -8mm of D] {Original\\ user opinions};
  \node (U) [right =5mm of D, utility] {Clicks};
  \node (P2) at (U|-P1) {Influenced\\ user opinions};

  \draw[->]
    (P1) edge (D)
    (P1) edge (P2)
    (D) edge (P2)
    (D) edge (U)
    (P2) edge (U)
  ;

  \iciincentive{D}
  \iciincentive[uchamf]{P2}
  \iciincentive[uchamf]{U}

    \cidlegend[left = 5.5mm of D,yshift=-1mm]{
      \node (3) [draw, circle, feasible control incentive] {};
      \node (4) [right = 0.5mm of 3, draw=none, rectangle] {ICI}; \\
}
  
\end{influence-diagram}
     \caption{
        Admits instrumental control incentive on user opinion
    }
    \label{fig:fci-application1}
  \end{subfigure}\hspace{3mm}
  \begin{subfigure}[t]{0.475\textwidth}
  \centering
  \begin{influence-diagram}
  \setrectangularnodes
  \setcompactsize

  \node (D) [decision] {Posts\\ to show};
  \node (Dp) [above left =21.5mm and -9mm of D,dashed] {Inert\\ posts};
  \node (P2p) [right=10.5mm of Dp,dashed] {Counterfactual\\ opinions};
  \node (P1) [above left =6mm and -13mm of D] {Original\\ user opinions};
  \node (U) [right =14mm of D, utility,dashed] {Hypothetical\\Clicks};
  \node (P2) [right=4mm of P1] {Influenced\\ user opinions};

  \draw[->]
    (P1) edge (D)
    (Dp) edge (P2p)
    (P1) edge (P2p)
    (D) edge (P2)
    (D) edge (U)
    (P1) edge (U)
  ;
  \path (P2p) edge[->, bend left=35] (U);

  \iciincentive{D}
  \iciincentive[uchamf]{U}

\end{influence-diagram}
       \caption{
      Admits no instrumental control incentive on user opinion
  }
  \label{fig:fci-application2}
  \end{subfigure}
  \caption[Content recommendation incentives]{In (a), the content recommendation example from \cref{fig:fci-preview} is shown to
      admit an instrumental control incentive on user opinion.
      This is avoided in (b) with a change to the objective.
      }\label{fig:fci-application}
\end{figure}

How could we instead design the agent to maximize clicks without manipulating the user's opinions
(i.e.\ without an instrumental control incentive on \emph{influenced user opinions})?
As shown in \cref{fig:fci-application2}, we could redesign the system so that
instead of being rewarded for the true click rate, it is rewarded for the
clicks that the user would give if they viewed 
some inert content that would not change their preferences.
An agent trained to maximize this objective would view any modification of user opinions as
irrelevant for improving its performance; however, it would still have an
instrumental control incentive for \emph{hypothetical clicks}, so it would still deliver desired content.


It is worth remarking on a possible challenge with identifiability, 
and how to address it.
Hypothetical clicks is a counterfactual variable, impossible to observe in reality (as in reality, users' behaviour is determined by their actual preferences).
More formally, it is defined using the third (i.e.\ counterfactual) rung of Pearl's causal hierarchy, 
and it relies on the ability to compute $\utilvars$ across different counterfactual worlds simultaneously, 
something that cannot be done by experiment without further assumptions \citep{avin2005identifiability}.
Fortunately, \citet{carroll2022estimating} demonstrate one set of natural assumptions under which the hypothetical clicks can be inferred from observed interactions with a user, essentially by inferring the (latent) user opinion variable from gradual shifts in user behaviour over longer sequences of interaction.


This example is an instance of a very wide class of safety worries, 
where some \emph{delicate variable} has an ICI \citep{farquhar2022path}.
\citet{omohundro2008basic} has hypothesised that an advanced AI system would have a \emph{convergent instrumental goal} to survive, or to obtain computing resources, 
which we may view as undesired ICIs.
\citet{armstrong2017oracles} has raised the concern that AI systems 
might seek to make self-fulfilling predictions, 
whereas we would not want them to manipulate the world.
Additionally, \citet{krueger2020hidden} have demonstrated that AI systems 
sometimes seek to induce shifts in the distribution of their 
testing data.
In each case, their proposed solution, as in our example,
is to impute a fixed value to the delicate variable.
Such a solution has been termed a \emph{path-specific objective}, 
because it requires the agent to optimise an objective, ignoring the effects of its 
decisions along some channels \citep{farquhar2022path}.
Intuitively, the agent is tasked with ``imagining that it cannot 
influence'' this delicate variable when choosing a decision.
For this to work, the variable must be robust to unintentional
influence, and when this will or will not be the case remains 
an open question for all of the examples discussed.

\section{Intent} \label{sec:intent}
Returning to the example from \cref{fig:fci-application}, 
we may want to ask a related question: 
assuming that the agent took a particular action which had a particular influence on the user, what was the reason that the agent took the action? 
Did it intend to influence the user in this way?
This is relevant for assigning blame and moral responsibility, among other things \citep{halpern2018towards}.

\citet{halpern2018towards} and \citet{ward2024reasons} operationalise `intent' by asking whether the agent would pick a different policy if it `knew' that the effect on some variables $\incentivevars$ (e.g.\ user opinions) was guaranteed.
Specifically, does there exist any suboptimal 
policy $\spi'$ that would surpass the performance of the agent's actual policy $\spi^*$ if the outcome of $\incentivevars$ was independent of its actions and fixed to ${\incentivevars_{\spi^*}}$?
This is necessary for the agent's
influence on $\incentivevars$ to be the actual cause of 
a policy's optimality \citep{ward2024reasons}.\footnote{
See \citep[Theorem 6]{ward2024reasons}, which shows that intent to cause an outcome 
is equivalent to the decision being an actual cause of the outcome.
}
If
$\incentivevars$ is a minimal set that satisfies 
this requirement, 
then the influence on that variable is said to be intentional.

There also exists an inverse question that has not 
been studied so far:
would the optimal policy
perform as badly as a suboptimal policy $\spi'$
if it only lost its control of $\incentivevars$ 
(i.e.\ if $\incentivevars$ were fixed to $\incentivevars_{\spi'})$?
Whereas the past definitions of intent pertain to ``adding'' control, 
this new question pertains to ``subtracting'' control, 
and allows us to define a new notion of intent.
The two ideas are unified in the definition below.

\begin{restatable}[Intent]{definition}{intent} \label{def:intent}
Let $\scim$ be a single-decision SCIM that represents an agent's
beliefs.
There is \emph{additive intent} 
to influence nodes $\sW$
by choosing $\spi^*$ over $\spi'$
if $\EE_{\spi'}[\totutilvar] < \EE_{\spi^*}[\totutilvar]$,
and $\sW$ is a subset $\sW \subseteq \sZ$ of variables $\sZ$, that is subset-minimal such that:
\begin{equation} \label{eq:reason-to-move}
\EE_{\spi'}[\totutilvar_{\sZ_{\spi^*}}] \geq \EE_{\spi^*}[\totutilvar]. 
\end{equation}
There is \emph{subtractive intent} if $\EE_{\spi'}[\totutilvar] < \EE_{\spi^*}[\totutilvar]$
and $\sZ$ is subset-minimal such that:
\begin{equation} \label{eq:reason-not-to-move}
\EE_{\spi^*}[\totutilvar_{\sZ_{\spi'}}] \leq \EE_{\spi'}[\totutilvar].
\end{equation}
For a set $\sPi'$, we say that there is an (additive/subtractive) intent 
to influence $\sW$ by choosing $\spi^*$ over $\sPi'$ if this intent is present over every $\spi'$ in $\sPi'$.
\end{restatable}

The notion of intent 
previously proposed in \citet{halpern2018towards} and \citet{ward2024reasons}
is equivalent to additive intent (\cref{app:intent}).
There is one difference in presentation:
since intent is about a policy newly reaching the level of another 
policy, this requires that their performances differ in the first 
place, so we have made explicit the
$\EE_{\spi'}[\totutilvar]<\EE_{\spi^*}[\totutilvar]$ condition 
that was implicit in the original definition.


Of these two notions, it is subtractive intent that comes closer to ICI, 
because it starts with the optimal policy $\spi^*$, as does intent,
and considers an intervention to $\incentivevars$ using an alternative 
policy $\spi'$.
Algebraically, the only difference 
is that the ICI indicates that this perturbation decreases performance a nonzero amount, 
while subtractive intent requires the perturbation to worsen 
performance beyond the threshold $\EE_{\spi'}[\totutilvar]$.
(Whereas additive intent starts from a suboptimal policy $\spi$, and is algebraically 
less similar.)
Both kinds of intent differ from ICI in that they evaluate an SCIM $\scim$, 
that corresponds to the agent's beliefs, rather than reality.
Despite these differences, both kinds of intent have the same 
graphical criterion as an ICI.
We can therefore generalize the graphical criterion from 
\citet{ward2024reasons} to accommodate both additive 
and subtractive intent.


\begin{restatable}[Intent Criterion]{theorem}{IntentCriterion}\label{theorem:intent-graph-criterion}
A single-decision CID $\cid$ admits (additive/ subtractive) intent on
$\incentivevars \subseteq \civars$ if and only if $\cid$ has a directed
path $D \pathto \incentivevar \pathto \utilvar$
for some $\incentivevar \in \incentivevars$
and $\utilvar \in \utilvars$.
\end{restatable}

Similarly to the ICI criterion, 
the intent criterion allows 
the agent to intend to influence \emph{clicks} 
and \emph{influenced user opinions}, 
whereas if the path-specific effect objective is used, 
then the agent can no longer intend to influence 
the user's preferences.

We can also make the relationship between intent and ICI more precise:
ICI is related to the presence of subtractive intent given optimal policies, 
although it is a slightly weaker condition, because it does not place any requirements 
on whether the alternative policy $\pi'$ must have a positive or negative influence on $U$ through $\sW$.

\begin{proposition}[Subtractive intent and ICI]
    In a single-decision SCIM $\calM$,
    if for all optimal $\pi^*$, there is subtractive intent to influence $\sW$ by choosing $\pi^*$ 
    over $\pi'$, then
    there is an ICI on $\sW$.
\end{proposition}

The proof is as follows.

\begin{proof}
We prove the result by contrapositive:
that if there is no ICI, 
then no optimal policy $\spi^*$ cannot satisfy both of the conditions for subtractive intent.

Let $\spi^*$ be an arbitrary optimal policy.
By the definition of ICI, we have that for all $\pa^D$,
$\EE_{pi^*}[\totutilvar_{\sW_d} \mid \pa^D] = \EE_{pi^*}[\totutilvar \mid \pa^D]$.
It follows that 
$\EE_{pi^*}[\totutilvar_{\sW_{\pi'}}] = \EE_{pi^*}[\totutilvar]$.
Recall that the conditions for subtractive intent are that:
$\EE^{\pi'}[\totutilvar]<\EE^{\pi^*}[\totutilvar]$
and $\EE^{\pi^*}[\totutilvar_{\sZ_{\pi'}}] \leq \EE^{\pi'}[\totutilvar]$.
But if both of these conditions were satisfieed, we would have
$$\EE^{\pi'}[\totutilvar]<\EE^{\pi^*}[\totutilvar] = \EE_{pi^*}[\totutilvar_{\sW_{\pi'}}] \leq \EE^{\pi'}[\totutilvar]$$
which is a contradiction, so there cannot be subtractive intent, proving the result.
\end{proof}

\section{Impact incentives} \label{sec:ai-acp:ii}
Even if an algorithm does not intentionally manipulate a 
sensitive variable, it may harmfully influence it 
unintentionally (i.e. as a side-effect).
For instance, even when a recommender system does not 
intend to manipulate human preferences, it may still do so
 \citep{jiang2019degenerate}.
This could be true if the persuasive videos are 
ones that the user prefers to click on even 
before any preference change has occurred.

To describe this kind of problem, 
we need a concept that checks whether the agent is impacting 
a variable relative to some baseline.
Formally, we can look at the assignments that this variable 
takes under the optimal policies, and evaluate their distance 
from the values that it assumes under some baseline policy, 
given a suitable distance metric.

\begin{definition}[Impact Incentive (II)]
Let $\incentivevars \subseteq \evars $ be nodes in a
  single-decision SCIM $\scim$.
There is an incentive to impact $\incentivevars$ 
with 
distance function
$\delta$ and threshold $c>0$, 
relative to baseline policy $\spi'$,
if
every optimal policy $\spi$ has
$\EE[\delta(\incentivevar_\spi(\seps),\incentivevar_{\spi'}(\seps)]>c$ for some assignment $\seps$.
  
  A CID $\causalgraph$ \emph{admits an impact incentive} if 
  there exists a model $\scim$, a distance function $\delta$, a $c\geq 0$ and a policy $\spi'$ such that there is an impact incentive.
\end{definition}

One way to think about this is that instead of asking whether the 
agent's influence on $\incentivevar$ is the reason that optimality 
is achieved (intent),
we are asking: does the constraint of optimality cause $\incentivevar$ to have a different distribution?

The graphical criterion is as follows.
\begin{restatable}[Impact Incentive Criterion]{theorem}{ImpactIncentiveCriterion}\label{theorem:ii-graph-criterion}
A single-decision CID $\cid$ admits an impact incentive on
$\incentivevars \subseteq \structvars$ if and only if some $\incentivevar \in \incentivevars$ and utility $U \in \utilvars$ are both descendants in 
$\cid$ of $D$.
\end{restatable}

In past work, it has been proposed to add a penalty term to
the objective of an AI system to reduce the impact 
on some variable $\incentivevar$, called an
\emph{impact measure} \citep{armstrong2017low,krakovna2018penalizing}.
Such proposals can be understood as constraining the 
size of the impact incentive in the following sense.
Consider an objective like $U + \lambda \delta(w,w')$ 
that encourages 
the AI system to keep $W$ close to some baseline value $w'$, 
according to some distance function $\delta$.
This objective will produce the smallest possible impact incentive, in terms of $\delta$,
for a given level of expected $\EE[U]$.
Graphically, an impact measure can be illustrated as in figure \cref{fig:twin-impact-measure}.
In this twin graph, \emph{counterfactual opinions} represents 
the baseline state from which distance is measured.
Then, \emph{impact measure} is computed as a function of $\incentivevar_{\spi'}$ and $\incentivevar_{\spi}$.
Adding impact measure as a new child of 
\emph{influenced user opinions} makes the AI care about this 
delicate variable.
Interestingly, this means that if a variable is impacted by 
a policy and then an impact measure is applied, there 
will be an ICI on that delicate variable --- the agent will 
try to control it, to keep it close to its baseline value.

\begin{figure}
      \centering
      \begin{influence-diagram}
  \setrectangularnodes
  \setcompactsize
  \tikzset{node distance=4mm and 4mm}
  \node (D) [decision, anchor=west] {Posts\\ to show};
  \node (Dp) [above left =16mm and -7mm of D,dashed] {Inert\\ posts};
  \node (P2p) [right=3.5mm of Dp,dashed] {Counterfactual\\ opinions};
  \node (R) [utility,below right=0.5mm and 5mm of P2p] {Impact \\ measure};
  \node (P1) [above left =4mm and -3mm of D] {Original\\ user opinions};
  \node (U) [right = of D, utility] {Clicks};
  \node (P2) at (U|-P1) {Influenced\\ user opinions};
\draw[->]
    (P1) edge (D)
    (P1) edge (P2)
    (P1) edge (P2p)
    (D) edge (P2)
    (D) edge (U)
    (P2) edge (U)

    (Dp) edge (P2p)
    (P2p) edge (R)
    (P2) edge (R)
  ;
\end{influence-diagram}
\caption[An impact measure, illustrated in a twin graph]{A twin graph depicting an impact measure.} \label{fig:twin-impact-measure}
\end{figure}

Similar identifiability issues arise as in the case of path-specific objectives discussed in \cref{sec:ai-acp:fci}: 
we are required to know the user's preferences in some counterfactual world.
In the case of impact measures, it is possible to avoid this problem by considering the KL divergence between $P_{\spi}(\incentivevars)$ and $P_{\spi'}(\incentivevars)$, rather than the distance between $\incentivevars_\spi(\seps)$ and $\incentivevars_{\spi'}(\seps)$.
The interventional distributions $P_{\spi}(\incentivevars)$ can be measured by experiment, which thereby avoids the counterfactual identifiability problem.


We will now compare and contrast the use cases of path-specific objectives 
versus impact measures.
If one is concerned with an agent intentionally manipulating a variable $\incentivevar$, 
then the agent's intent is the problem.
For example, we may worry about a content recommender intentionally altering users 
preferences. 
In this case, the intent (and ICI) may be removed with a path-specific objective 
\citep{farquhar2022path}, as shown in \cref{fig:fci-application}.
This will allow the variable $\incentivevar$ to drift from its original value, 
as a side-effect of AI action, or for other reasons altogether.
For example, users may still discover new interests that change their preferences, 
and we may regard this as desirable, so long as it is not a result of manipulation by 
the AI.
In other cases, we may have in mind a clear specification for how $\incentivevar$
should behave, and want to prevent any drift, intentional or otherwise,
from this baseline value.
For example, we may worry that users are led to political extremism, not because of the 
content recommender, but rather because of politically-motivated content creators, 
and we want our content recommender to actively defend against this by 
suppressing such content.
In this case, an impact measure \citep{krakovna2018penalizing} is more appropriate, 
and will limit impact incentive on users' preferences.%
\footnote{One other possible remedy would be ``quantilisation'' \citep{taylor2016quantilizers}, which 
seeks a policy with that is similar a trusted baseline, in terms of 
a guaranteed upper bound on the Kullback-Leibler divergence.
We may wish to say that quantilisers upper-bound the impact incentives, 
on the variable $W=D$,
where $\delta$ is the Kullback-Leibler divergence.
However, Kullback-Leibler divergence is a function of the distribution, 
$P_\spi(\seps)$ rather than particular assignments $W_\spi(\seps),W_{\spi'}(\seps)$.
Perhaps this connection could be spelled out by defining impact incentives in a 
causal influence diagram (i.e.\ rung-2) setting, but this matter is left to future 
work.
} %
It is important to note that the presence of instrumental control assumptions 
can be sensitive to the modelling assumptions used to analyse an agent.
For example, consider an RL agent that uses Q-learning to solve an environment with two timesteps.\footnote{Thanks to Paul Christiano for this example.}
It is natural to model this Q-learner as a single agent as in \cref{fig:q-learner1}, 
where $D$ is chosen to optimise the reward $R$.
Then, the future state $s'$ satisfies the instrumental control incentive criterion.
This matches our intuition --- that RL systems may benefit from shaping their future environment.
Suppose instead that we regard as an agent
the function inside the Q-learner that chooses $d$ to maximize the \emph{Q-function} $q(s,d):=\EE[R \mid s,d]$.
In this model, shown in \cref{fig:q-learner2}.%
\footnote{It would also be possible to consider a multi-agent influence diagram \citep{hammondreasoning} where the $Q$ function is included as a decision, and 
its goal is a loss function $\ell = \lvert r-\hat{r} \rvert$, but we note that the set of variables 
that satisfy the graphical criterion for an ICI would not be altered by including this $Q$ variable, along 
with a utility variable $\ell$ that is a child of $Q$ and $R$.}
Then, the decision's effect on $s'$ is a mere side-effect to the task of maximizing $q(s,d)$.
Although the instrumental control incentive is absent, 
the physical reality of this second scenario is identical to the first, and so
there any harmful influence on $s'$ may still be finely tuned to the agent's objective.

Ideally, we would reduce this sensitivity to modelling assumptions, 
and we might hope to achieve this by using more fundamental modelling assumptions, 
such as the independent causal mechanism assumption, to ascertain
which variables should be viewed as decisions \citet[Sec.~4.3]{kenton2023discovering}. 
But such approaches still are sensitive to which variables are regarded as 
causal mechanisms or physical variables, and further research is needed 
to understand this dependence.



\begin{figure}
\begin{subfigure}[t]{0.37\linewidth}
    \centering
  \begin{influence-diagram}[node distance=5mm]\setcompactsize
    \node (S) {$S$};
    \node (Sprime) [right=16mm of S] {$S'$};
    \node (D) [below right= 7.5mm and 7mm of S] [decision] {$D$};
    \node (Q) [below=19mm of S,phantom] {\phantom{$Q$}};
    \node (R) [utility,right=16mm of Q] {$R$};

    \draw[->] (S) -- (Sprime);
    \draw[->] (Sprime) -- (R);
    \draw[->] (S) -- (D);
    \draw[->] (D) -- (Sprime);
    \draw[->] (D) -- (R);

    \iiincentive[inner sep=-0.3mm]{Sprime}
    \iciincentive[inner sep=.4mm]{Sprime}
    \iiincentive[uchamf=2.5mm, inner sep=-1.3mm]{R}
    \iciincentive[uchamf=2.5mm, inner sep=-0.6mm]{R}
    \end{influence-diagram}
        \caption{A one-step decision problem}
    \label{fig:q-learner1}
  \end{subfigure}
  \vspace{2mm}
  \begin{subfigure}[t]{0.46\linewidth}
  \centering
    \begin{influence-diagram}[node distance=7mm]\setcompactsize
    \node (S) {$S$};
    \node (Sprime) [right=20mm of S] {$S'$};
    \node (D) [below right= 3mm and 8.5mm of S] [decision] {$D$};
    \node (Rhat) [utility,below=8mm of D,label={[yshift=-2mm,shape=rectangle]below:$f_{\hat{R}}(s,d)\!=\!\EE[R \mid s,d]$}] {\hspace{-0.5mm}\vspace{-0.3mm}$\hat{R}$};
    \node (R) [right=6mm of Rhat] {$R$};

    \draw[->] (S) -- (Sprime);
    \draw[->] (S) -- (D);
    \draw[->] (Sprime) -- (R);
    \draw[->] (D) -- (Sprime);
    \draw[->] (D) -- (R);
    \draw[->] (D) -- (Rhat);
    \draw[->] (S) -- (Rhat);

    \iiincentive[uchamf=2.5mm, inner sep=-1.3mm]{Rhat}
    \iciincentive[uchamf=2.5mm, inner sep=-0.6mm]{Rhat}
    \iiincentive[inner sep=-0.3mm]{Sprime}
    \iiincentive[inner sep=-0.3mm]{R}
    \end{influence-diagram}
        \caption{The task of maximizing a Q-function}
  \label{fig:q-learner2}
  \end{subfigure}
      \begin{subfigure}[t]{0.15\textwidth}
    \begin{influence-diagram}
    \cidlegend{
      \legendrow{impact incentive}{II} \\
      \legendrow{feasible control incentive}{ICI} \\
}\end{influence-diagram}\end{subfigure}
  \caption[Incentives of a Q-learner]{Two possible representations of a Q-learner solving a one-step decision problem.}\label{fig:q-learner}
\end{figure}

\section{Incentives in a multi-decision setting} \label{sec:ai-acp:multi-decision}
There are multiple possible ways that incentive concepts like RI, ICI and II may be generalized 
to multi-decision settings.
This is because the presence of an incentive at some decision $D$
may depend on the policy followed at other decisions.
If we want to know the incentives when a model is fully trained, we could 
see whether some incentive concept $\phi$ holds
for some or all of the optimal policies.
Alternatively, we may be interested in sub-optimal policies as well.
Both cases are included in the following definition.
Note that in the case of a multi-decision CID, we will denote
the decision rule for a particular decision $D^i$ as $\pi^i$, 
and the set of decision rules for all other policies as $\spi^{-i}$.%
\footnote{Those familiar with temporal logic in games may notice that this is analogous to the 
notion of E-NASH and A-NASH propositions --- 
    ones that hold in one or all Nash Equilibria, respectively \citep{chatterjee2010strategy,wooldridge2016rational}.}

\begin{figure}
\centering
  \begin{influence-diagram}\setcompactsize
    \node (D1) [decision] {$D$};
    \node (S1) [right=10mm of D1]{$S$};
    \node (D2) [decision, below=10mm of D1] {$D'$};
    \node (S2) [right=10mm of D2] {$S'$};
    \node (C) [right=10mm of S1] {$C$};
    \node (U) [utility,right=8.3mm of S2] {$U$};

    \draw[->] (D1) -- (S1);
    \draw[->] (D1) -- (D2);
    \draw[->] (D2) -- (S2);
    \draw[->] (S1) -- (U);
    \draw[->] (S2) -- (U);
    \draw[->] (C) -- (U);

    \end{influence-diagram}
        \caption{The task of opening a combination lock} \label{fig:ai-acp:combo-lock}
\end{figure}

\begin{definition}[Multi-decision $\phi$-incentive]
Let $\phi$ be a proposition defined on a single-decision SCIM, 
and let $\scim$ be a multi-decision SCIM.
    There is an A- (resp.\ E-) optimal $\phi$ at the decision 
    $\decisionvar^i$ if for all (resp. there exists some) $\spi \in \argmax_{\spi'} \mathbb{E}_{\spi'}[U]$, such that
    $\phi$ holds in $\scim_{\spi^{-i}}$,
    the single-decision SCIM obtained by substituting in the decision rules 
    $\spi^{-i}$ for decisions other than $\decisionvar^i$ into $\scim$.
    
    There is an A- (resp.\ E-) pre-optimal $\phi$ at $\decisionvar^i$ if
    for all (resp.\ there exists some) $\spi \in \sPi$, such that
    $\phi \text{ holds in } \scim_{\spi^{-i}}$,
    where $\sPi$ is the set of all policies.
\end{definition}
We focus exclusively on cases where $\phi$ is the presence of a RI, II, or ICI, in a single-decision SCIM.

For example, consider the task of opening a combination lock (\cref{fig:ai-acp:combo-lock}).
Assume that the correct combination is $c=(9,9)$.
There are two decisions, $d,d' \in \{0,\ldots,9\}$, 
which are stored in the states $s=d$ and $s'=d'$,
and that are checked against the combination to output utility of $1$ or $0$, 
i.e.\ $u = \delta(c[1]=s \land c[2]=s')$.
If $D$ is chosen optimally, i.e.\ $d=9$, then $S'$ has an instrumental control incentive for $D'$, 
because it must be set to $9$ in order to obtain $u=1$.
In other words, an A-optimal instrumental control incentive is present.
If instead $D$ is set to $8$, then $D'$ lacks any such incentive.
So there is no A-pre-optimal instrumental control incentive.

The fact that the instrumental control incentive is present for all optimal policies
implies that it is also present for one optimal policy, i.e.\ that an E-optimal 
incentive is present, and for one policy altogether, i.e.\ that an E-pre-optimal incentive 
is also present.
This is a general rule: the four types of multi-decision incentive 
always have this inclusion relation.
\begin{proposition} \label{prop:hierarchy}
    For any $\phi$, 
    A-pre-optimal incentive $\implies$
    A-optimal incentive $\implies$
    E-optimal incentive $\implies$
    E-pre-optimal incentive.
\end{proposition}
\begin{proof}
    These implications, from left to right, hold because:
    i) Optimal policies are a subset of all policies, ii) any optimal policy is in the set of optimal policies, and iii) any optimal policy is a policy.
\end{proof}

In establishing graphical criteria for these incentive concepts, 
we can draw on a helpful equivalence.
An E-pre-optimal incentive on $D^i$ is equivalent to 
compatibility with a single-decision incentive on $D^i$, 
treating other decisions as chance variables.
To see this, notice that in either case,
one can impute any function to variables other than $D^i$.
Since an E-pre-optimal incentive is the weakest of the four kinds of multi-decision incentive, the single-decision graphical criteria 
can be used to rule out \emph{any} form of multi-decision incentive.
\begin{proposition}
Let $\scim$ be a multi-decision SCIM, 
and obtain $\scim'$ by replacing all decisions except for $D^i$ with chance nodes.
If the graphical criterion for single-decision (RI/ICI/II) does not hold in $\scim'$,
then there is no A- or E-optimal or pre-optimal multi-decision (RI/ICI/II) in $\scim$. 
\end{proposition}
\begin{proof}
    Immediate from \cref{prop:hierarchy} and the fact that choosing a set of functions and distributions $\{f^j,P^j\}$ for $D^{-i}$ such that $\scim_{\{f^j,P^j\}_{j \neq i}}$ satisfies $\phi$
    is equivalent to choosing a set of deterministic decision rules and distributions $\{\pi^j,P^j\}$ for $D^{-i}$ such that $\scim_{\{\pi^j,P^j\}_{j \neq i}}$ satisfies $\phi$.
\end{proof}


\section{Related work} \label{sec:ai-acp:related-work}

\paragraph{Causal influence diagrams}
The use of structural functions in a causal influence diagram 
goes back to at least
the \emph{functional influence diagram} of \citet{dawid2002influence}.
The most similar alternative model is the Howard canonical form influence diagram \citep{howard1990influence,heckerman1995decision}.
However, this only permits counterfactual reasoning downstream of decisions, which is
inadequate for defining the response incentive.
Similarly, the causality property for influence diagrams introduced by
\citet{heckerman1994decision} and \citet{shachter2010pearl} only constrains the relationships to being partially causal, 
in that decisions are taken to be causally antecedent to their descendants
(though adding new decision node parents to all nodes makes the diagram fully causal). \Cref{app:causality-examples} shows by example why the stronger causality
property
is necessary for most of the newly proposed incentive concepts.
Building on this paper, multi-agent SCIMs are formalized in \citet{hammondreasoning}, and
an open-source Python implementation of CIDs has been
developed 
\citep{fox2021pycid}.

\paragraph{Materiality and value of information} 
The criterion for materiality, \Cref{th:soft-sa},
 builds on previous work.
The concept of value of information was first introduced by \citet{howard1966information}.
The materiality soundness proof follows previous proofs \citep{Shachter1998,lauritzen2001representing},
while the completeness proof is most similar to an attempted proof
by \citet{nielsen1999welldefined}. They propose the criterion $\incentivevar \not\perp \utilvars^D\mid \Pa_D$ for requisite
nodes, which differs from \cref{eq:voi-criterion} in the conditioned set.
Taken literally,\footnote{Def.~\ref{def:ai-acp:d-separation} defines d-separation for potentially
overlapping sets.}
their criterion is unsound for requisite nodes.
For example, in \cref{fig:race-a}, \emph{high school} is d-separated from
\emph{accuracy} given $\Pad$, so their criterion would
fail to detect that \emph{high school} is requisite and admits VoI.\footnote{Furthermore, to prove that nodes meeting the d-connectedness property are requisite,
\citeauthor{nielsen1999welldefined} claim that
``$X$ is [requisite] for $D$ if $\Pr(\dom(U)\mid D,\Pad)$ is a function of $X$ and $U$ is a utility function relevant for $D$''.
However, $U$ being a function of $X$ only proves that $U$ is conditionally dependent on $X$,
not that it changes the expected utility, or is requisite or material.
Additional argumentation is needed to show that conditioning on $X$ can actually change the expected utility;
our proof provides such an argument.
Since an earlier version of this paper was placed online \citep{everitt2019understanding},
this completeness result was independently discovered by
\citet[Thm. 2]{zhang2020causal} and \citet[Thm. 1]{lee2020characterizing}.
There has also been further work in generalising this result to the case of multi-decision influence diagrams, in \citet{van2022complete}, where a 
sound and complete criterion is known for a class of influence diagrams
said to satisfy ``solubility'', also known as ``sufficient recall''.
} 

To have positive VoC, it is known that a node must be an ancestor of a utility node \citep{shachter1986evaluating},
but the authors know of no more specific criterion.
The concept of a \emph{relevant} node introduced by
\citet{nielsen1999welldefined} also bears some resemblance to VoC.

The relation of the current technical results to prior work is summarised in \cref{table:prior-work}.

\begingroup
\newcommand*{\cellcite}[1]{\citet{#1}}
\rowcolors{2}{gray!20}{white}
\begin{table*}[!th]
    \centering


\ifthesis
\newcommand{\extrawidth}{4mm}
\else
\newcommand{\extrawidth}{0mm}
\fi
\begin{tabular}{p{1.58cm+\extrawidth} L{2.15cm+\extrawidth} L{2.35cm+\extrawidth} L{2.4cm+\extrawidth} L{2.2cm+\extrawidth}}
& Definition & Criterion & Soundness &  Completeness \\
\midrule
\raggedright Mater-iality &
\cellcite{howard1966information};
\cellcite{matheson1990using} &
\cellcite{fagiuoli1998note};
\cellcite{lauritzen2001representing};
\cellcite{shachter2016decisions} &
\cellcite{fagiuoli1998note};
\cellcite{lauritzen2001representing};
\cellcite{shachter2016decisions} &
First correct proof to our knowledge; see \cref{sec:ai-acp:related-work}
\\
RI & New & New & New; proved using do-calculus and d-sep & New; proved constructively \\
ICI & New & New & New; proved using do-calculus & New; proved constructively \\
\raggedright (Positive/ negative) intent & (\cellcite{halpern2018towards}/new) & (\cellcite{ward2024reasons}/ new) & (\cellcite{ward2024reasons}/ new) & (\cellcite{ward2024reasons}/ new) \\
II & New & New & New; proved using do-calculus & New; proved constructively \\
\end{tabular}
\caption{
  Comparison with previous work, in a single-decision setting.
  The concept of materiality is well-known. 
    For VoI, a new, corrected proof is provided.
    For VoC, the present work offers a new criterion, proving it sound and complete.
  For response incentive (RI) and instrumental control incentive (ICI), 
  the criterion and all proofs are new.
}\label{table:prior-work}
\end{table*}
\endgroup

\paragraph{Instrumental control incentives and intent} 
In a causal setting, \citet{kleiman2015inference} offered a notion of intention to influence a 
variable $O$. 
A different kind of approach was taken by \citet{halpern2018towards} and \citet{ward2024reasons}, which offered 
definitions of intent that are specific to outcomes $O=o$.
In particular, \citet{ward2024reasons} was the first to prove a graphical criterion for any version of intent.
We extend this work by defining a positive version of intent, rather than just considering negative intent, 
and by proving a graphical criterion for this new concept.

\ryan{Is there anything to say about IIs and impact measures? Quantilisers? Skalse's mild optimisation thing?}

\paragraph{AI fairness}
Another application of this work is to evaluate when an AI system is incentivized to behave unfairly,
on some definition of fairness.
Response incentives address this question for counterfactual fairness \citep{kusner2017counterfactual,kilbertus2017avoiding}.
An incentive criterion corresponding to path-specific effects \citep{zhang2016causal,nabi2018fair} has been established by 
\citet{ashurst2022fair}, for the single-decision setting.
\citet{nabi2019learning} have shown how a policy may be chosen subject to path-specific effect constraints.
However, they assume recall of all past events, whereas the response incentive criterion applies to any CID.~\looseness=-1

\paragraph{Rational Verification}
Verification is the task of checking that a program satisfies specified properties, 
which is relevant to the present study because we are proposing to use incentive concepts to check agent behaviour.
Typically, specifications are defined using temporal logic \citep{emerson1990temporal}; sometimes a probabilistic temporal logic is used \citep{kwiatkowska2022probabilistic}. 
Of particular relevance is ``rational verification'', which validates the behaviour of agents that are pursuing objectives \citep{abate2021rational,gutierrez2021rational,wooldridge2016rational}. 
Overall, our work aligns with rational verification in that it aims to verify agent behaviour. 
The difference is that we have explored what kinds of properties can be specified in the language of causality in particular (rather than, for example, a temporal logic).
Relatedly, rather than using a Kripke structure of partially observable Markov decision process to model an agent-environment interaction, we have used causal models.

\paragraph{Mechanism design}
The aim of mechanism design is to understand how objectives and environments can
be designed, in order to shape the behaviour of rational agents (e.g.\ \citealp[Part II]{nisan2007algorithmic}).
At this high level, mechanism design is closely related to the incentive
design results we have developed in this paper. In practice, however, the strands of research look rather different.
Whereas mechanism design is primarily concerned with defining objective functions and action spaces 
that ensure desirable Nash equilibria, 
our core interest is on defining specifications for safe and fair agent behaviour, 
and on the causal structures that ensure that these specifications are satisfied.
\enlargethispage{\baselineskip}

\section{Discussion and conclusion}\label{sec:ai-acp:conclusion}
We have defined 
three new concepts: response incentives, instrumental control incentives 
and impact incentives, 
and have spelled out the connection between ICIs 
and an existing concept, intent.
We have proved complete graphical criteria for all four concepts 
in a single-decision setting.
Moreover, we have introduced a notion of incentives for influence diagrams with multiple decisions, 
and proved that the criteria are also sound for those cases.
In all cases we have shown how these definitions have implications 
for other concepts of broader interest, such as instrumental goals, 
counterfactual fairness, and impact measures.
We have also shown via toy examples how different existing approaches 
might be appropriate to addressing different kinds of problems, 
and have outlined circumstances in which each kind of approach is favoured.
These incentive concepts have already seen applications 
to areas including 
value learning \citep{armstrong2020pitfalls}, 
interruptibility \citep{langlois2021rl}, conservatism \citep{cohen2020unambitious}, 
modelling agent frameworks \citep{everitt2019modeling} and reward tampering \citep{everitt2021reward}.

Let us now outline some limitations of this paper, and what they might mean 
for future work.
First, note that to apply these criteria, we 
require knowledge of the (causal) structure of the interaction between
agent and environment.
Sometimes, experts know these causal relationships even when they do not know 
the exact parametric relationships between variables --- 
an ideal use case for these criteria.
In the context of incentive design, such a scenario may often arise, since
these causal relationships often follow directly from the design choices for 
an agent and its objective.
Sometimes, however, we may have too little knowledge of the causal structure 
to be able to apply the criteria.
In other cases, we may have, in a sense, too much knowledge for the graphical criteria to be 
useful. 
With abundant experimental data, we might compute safety and fairness properties (such as 
counterfactual fairness) directly, removing any need for the incentive concepts and graphical 
criteria.
A fourth scenario is that the world is not even describable by a fixed graphical model, 
but rather it is better understood using a probability tree, or relatedly, as an
extensive form game.
These limitations suggest possible avenues for future work.
To enlarge the set of cases in which 
incentives can be evaluated, it may be possible to devise ways of
combining experimental data with a priori knowledge to arrive at an evaluation.
To deal with extensive form games, it may be possible to devise graphical criteria 
for probability tree and game trees.

Another limitation of graphical criteria is that they can only offer a definitive 
resolution in one direction.
Also, although they can rule out incentives definitively, they can only 
rule that the presence of an incentive is compatible with the graphical structure.
It is still yet to be established how often incentives are present when they 
are compatible with the graph.
This might be proved using measure theoretic arguments resembling the arguments 
that d-connection almost always implies conditional dependence \citep{meek1995strong}.
Relatedly, their output says nothing of the strength of incentive present, 
which can only be established using detailed knowledge of the strength of causal relationships 
present in the environment, rather than just their presence or absence.

Finally, it would be possible to improve the applicability of these graphical criteria 
by extending them to multi-agent settings.
So far, we have considered single-agent settings, 
where the world is divided into agent and environment.
If instead part of the environment was modelled as a rival agent, 
and we assume Nash Equilibrium policy profiles, then this would place additional 
constraints on how that part of the environment may behave.
So, in some cases where single-agent criteria cannot rule out an incentive, 
a multi-agent criterion should be able to rule out that incentive.
On the other hand, if it is known that another player will observe and respond strategically 
to one's policy, then this could mean that policies could influence one another via pathways 
that are not visible in the original causal graph, which could mean that multi-agent 
incentives might arise, when the criteria for a single-agent setting would have ruled them 
impossible.
Some groundwork has been laid by \citep{hammondreasoning}, which formalizes multi-agent influence 
diagrams, but a full analysis of the multi-agent setting is left to future work.

\clearpage
\bibliography{refs.bib}
\appendix 
\section{Causality Examples}
\label{app:causality-examples}

\begin{figure}[H]
  \centering
  \begin{subfigure}{0.48\linewidth}
    \begin{influence-diagram}
      \node (D) [decision] {$D$};
      \node (U) [right = of D] [utility] {$U$};
      \node (W) [above = of U] {$W$};

      \edge {D} {W};
      \edge {D, W} {U};

    \begin{scope}[
      node distance = 1mm,
      every node/.style = {rectangle, draw=none}]
      \small
        \node [above = of W] {$W=D$};
        \node [below = of U, xshift=3mm] {$U=W+D$};
        \node [below = of D, xshift=-3mm] {$D\in\{0,1\}$};
      \end{scope}
    \end{influence-diagram}
    \caption{A causal influence diagram reflecting the causal structure of the
      environment}
    \label{fig:causality-example-a}
  \end{subfigure}
  \hspace{1mm}
    \begin{subfigure}{0.48\linewidth}
    \begin{influence-diagram}
      \node (D) [decision] {$D$};
      \node (U) [right = of D] [utility] {$U$};
      \node (W) [above = of U] {$W$};

      \edge {D} {W};
      \edge {D} {U};
    \begin{scope}[
      node distance = 1mm,
      every node/.style = {rectangle, draw=none}]
      \small
        \node [above = of W] {$W=D$};
        \node [below = of U, xshift=3mm] {$U=2\cdot D$};
        \node [below = of D, xshift=-2mm] {$D\in\{0,1\}$};
      \end{scope}
    \end{influence-diagram}
    \caption{Influence diagram that is causal in the sense of
      \citet{heckerman1994decision,heckerman1995decision}}
    \label{fig:causality-example-b}
  \end{subfigure}
  \caption{Two different influence diagram representations of the same
    situation, with different VoC and ICI.}
  \label{fig:causality-example}
\end{figure}

\begin{figure}[H]
  \centering
  \begin{subfigure}{0.475\linewidth}
    \begin{influence-diagram}
      \small
      \node (D) [decision] {$D$};
      \node (W) [above = 5mm of D] {$W$};
      \node (Y) [left = 5mm of W] {$Y$};
      \node (U) [right = 5mm of D] [utility] {$U$};

      \edge {Y} {W};
      \edge {W} {D};
      \edge {D, W} {U};

      \begin{scope}[
          node distance = 1mm,
          every node/.style = {rectangle, draw=none}]
        \small
        \node [above = of Y, xshift=-3mm, yshift=-0.6mm] {$Y\!\sim\! \{0, 1\}$};
        \node [above = of W, xshift=7mm] {$W=Y$};
        \node [below = 1mm of U, xshift=3mm] {$U\!=\!W\!+\!D$};
        \node [below = 1mm of D, xshift=-4.5mm] {$D\in\{0,1\}$};
      \end{scope}
    \end{influence-diagram}
    \caption{A causal influence diagram reflecting the causal structure of the
    environment}\label{fig:causality-example-ri-a}
  \end{subfigure}
  \hspace{1mm}
  \begin{subfigure}{0.475\linewidth}
    \begin{influence-diagram}
      \small
      \node (D) [decision] {$D$};
      \node (W) [above = 5mm of D] {$W$};
      \node (Y) [left = 5mm of W] {$Y$};
      \node (U) [right = 5mm of D] [utility] {$U$};

      \edge {W} {Y};
      \edge {W} {D};
      \edge {D, W} {U};

      \begin{scope}[
          node distance = 1mm,
          every node/.style = {rectangle, draw=none}]
        \small
        \node [above = of W, xshift=7mm, yshift=-0.6mm] {$W\!\sim\! \{0, 1\}$};
        \node [above = of Y, xshift=-3mm] {$Y=W$};
        \node [below = of U, xshift=3mm] {$U\!=\!W\!+\!D$};
        \node [below = of D, xshift=-4.5mm] {$D\in\{0,1\}$};
      \end{scope}
    \end{influence-diagram}
    \caption{Influence diagram that is causal in the sense of
    \citet{heckerman1994decision,heckerman1995decision}}\label{fig:causality-example-ri-b}
  \end{subfigure}
  \caption{Two different influence diagram representations of the same
    situation, with different RI and VoC.
    In \cref{fig:causality-example-ri-a},
    $Y$ is sampled from some arbitrary distribution on $\{0, 1\}$, for example a
    Bernoulli distribution with $p=0.5$.
    In \cref{fig:causality-example-ri-b}, $W$ is sampled in the same way.
}
  \label{fig:causality-example-ri}
\end{figure}

Causal influence diagrams that reflect the full causal structure of
the environment are needed to correctly capture response incentives,
value of control and instrumental control incentives.
We begin with showing this for instrumental control incentives and value of control,
leaving response incentive to the end of this section.
Consider the two influence diagrams in \cref{fig:causality-example}.
If we assume that $W$ really affects $U$, only the diagram in
\cref{fig:causality-example-a} correctly represents this causal structure, whereas
\cref{fig:causality-example-b} lacks the edge $W\to U$.
According to \cref{def:ai-acp:instrumental-goal,def:ai-acp:control-incentive-sa}, $W$ has
positive value of control and an instrumental control incentive.
Only \cref{fig:causality-example-a} gets this right.

The influence diagram literature has discussed weaker notions of causality,
under which \cref{fig:causality-example-b} is considered a valid alternative
representation of the situation described by \cref{fig:causality-example-a}.
For example, if we only consider their joint distributions conditional on
various policies, then
\cref{fig:causality-example-a,fig:causality-example-b}
are identical.
Both diagrams are also in the canonical form of
\citet{heckerman1995decision}, as every variable responsive to the
decision is a descendant of the decision.
For the same reason,
both diagrams are also causal influence diagrams in the terminology of
\citet{heckerman1994decision} and \citet{shachter2010pearl}.
Since only \cref{fig:causality-example-a} gets the incentives right, we see that
the stronger notion of causal influence diagram introduced in this paper is
necessary to correctly model instrumental control incentives and value of control.

To show that response incentives also rely on fully causal influence diagrams,
consider the diagrams in \cref{fig:causality-example-ri}.
Again, we assume that \cref{fig:causality-example-ri-a}
accurately depicts the environment, while
\cref{fig:causality-example-ri-b} has the edge $Y\to W$ reversed.
Again, both diagrams have identical joint distributions
given any policy.
Both diagrams are also causal in the weaker sense of
\citet{heckerman1994decision} and \citet{shachter2010pearl}.
Yet only the fully causal influence diagram in \cref{fig:causality-example-ri-a}
exhibits that $Y$ can have a response incentive or positive value of control.

\section{Value of Information} \label{app:ai-acp:voi}
Materiality can be generalized to nodes not observed,
to assess which variables a decision-maker would benefit from
knowing before making a decision, i.e.\ which variables have value of information \citep{howard1966information,matheson1990using}.
To assess VoI for variables $\incentivevars$,
we first make $\incentivevars$ an observation by adding a link $\incentivevar \to D$ for each $\incentivevar \in \incentivevars$
and then test whether any $\incentivevar$ is material in the updated model \citep{shachter2016decisions}.

\begin{definition}[Value of information] \label{def:ai-acp:observation-incentive-sa}
  Nodes $\incentivevars \subseteq \evars \setminus \Descv{\decisionvar}$ in a
  single-decision SCIM $\scim$ have \emph{VoI}
  if 
  $\attutil(\scim_{\incentivevar \not \to D})
  < \attutil(\scim_{\incentivevar \to D})
  $
  where $\scim_{\incentivevars \to D}$ 
  is obtained from $\scim$ by adding the edges 
  from each $\incentivevar \in \incentivevars$ to $D$, 
  and $\scim_{\incentivevars \not \to D}$ is obtained 
  by removing them.
\end{definition}
Since \cref{def:ai-acp:observation-incentive-sa} adds an information link, it can only be applied to variables $\incentivevars$ that are
non-descendants of the decision, lest cycles be created in the graph.

We will say that a CID $\causalgraph$ \emph{admits} VoI for $\incentivevars$ if $\incentivevars$ has VoI in a
  a SCIM $\scim$ compatible with $\causalgraph$. 
  More generally, for any proposition $\phi$, we will say that 
  $\causalgraph$ admits $\phi$ if 
  there exists any SCIM $\scim$ compatible with $\causalgraph$ 
  that satisfies $\phi$.


An observed variable having positive VoI means that it would 
be material if it was observed.
Using this insight, we can adapt the criterion from \cref{def:ai-acp:requisite-observation} to check for positive VoI.
For a latent variable, we add an edge from it to the decision, 
and then check the graphical criterion.
We prove that this procedure is tight, in that it identifies 
every zero VoI node that can be identified 
from the graphical structure (in a single decision setting).

\begin{restatable}[Value of information criterion]{theorem}{ValueOfInformationCriterion}
  \label{th:voi}
  A single decision CID $\causalgraph$ admits VoI for $\incentivevars \subseteq \evars \setminus \Descv{\decisionvar}$
    if and only if there exists some $\incentivevar \in \incentivevars$ that is a requisite observation in $\causalgraph_{\incentivevars \to D}$, the graph 
    obtained by adding edges from $\incentivevars$ to $D$, to $\causalgraph$.
\end{restatable}

The proof is deferred to \Cref{app:ai-acp:appendix-voi}.

\section{Value of Control} \label{app:ai-acp:voc}
\label{sec:ai-acp:control-incentives}
\enlargethispage{\baselineskip}

So far, we have considered what information an agent would like 
to know, or be influenced by.
We now consider what variables an agent would like to control.
A variable has VoC if a decision-maker could benefit from setting its value \citep{shachter1986evaluating,matheson1990using,shachter2010pearl}.
Concretely, we ask whether the attainable utility can be increased by
letting the agent decide the structural function for the variable.

\begin{definition}[Value of control]\label{def:ai-acp:control-incentive-sa}
  In a single-decision SCIM $\scim$, the set of non-decision nodes
    $\incentivevars$ has \emph{positive value of control}
  if
  \[
    \max_{\pi}\EEs{\pi}[\totutilvar]
    <
    \max_{\pi, \gsw}\EEs{\pi}[\totutilvar_{\gsw}]
    \]
    where $\gsw$ is a set of soft interventions for $\incentivevars$,
    i.e. a new structural function $g^\incentivevar:\sfsig{\incentivevar}$ that respects the graph, 
    for each $\incentivevar \in \incentivevars$.
\end{definition}

This can be deduced from the graph, using again the minimal reduction
(\cref{def:ai-acp:reduced-graph}) to rule out effects through observations that an
optimal policy can ignore.~\looseness=-1

\begin{restatable}[Value of control criterion]{theorem}{ValueOfControlCriterion}
  \label{th:soft-sa}
  A single-decision CID $\causalgraph$ admits positive value of control for 
  non-decision vertices $\incentivevars \subseteq \evars \setminus \{D\}$
  if and only if there is a directed path
  $\incentivevar \pathto \utilvar$ for some $\incentivevar \in \incentivevars$ and $\utilvar \in \utilvars$
  in the minimal reduction $\reducedgraph$.
\end{restatable}
 
The proof is supplied in \Cref{app:ai-acp:appendix-voc-criterion}.

To apply this criterion to the content recommendation example (\cref{fig:fci-application1}), we first obtain the minimal reduction, which is identical
to the original graph.
Since all non-decision nodes are upstream of the utility in the minimal reduction, they all admit positive VoC.
Notably, this includes nodes like \emph{original user opinions} and \emph{model of user opinions}
that the decision has no ability to control according to the graphical structure.
In the next section, we propose \emph{instrumental control incentives}, which
incorporate the agent's limitations.~\looseness=-1

\section{Intent Equivalence} \label{app:intent}
First, let us restate our definition.
\intent*

And here is Halpern's definition, translated into an SCIM setting.

\begin{definition}[Intent; adapted from Def.\ 4.4 of \citep{halpern2018towards}]
In a single-decision SCIM $\calM$,
an agent intends to affect $\incentivevars$ by choosing policy $\spi$
and reference set $\sPi'$
if there exists a superset $\sZ \supseteq \incentivevars$ such that:
a) $\EE[\totutilvar_\spi] < \max_{\spi'} \EE[\totutilvar_{\spi',\sZ_{\spi}}]$, and
b) $\sZ$ is subset-minimal; i.e.\ for any strict subset $\sZ^*$, we have 
$\EE[\totutilvar_\spi] \geq \max_{\spi'} \EE[\totutilvar_{\spi',\sZ^*_{\spi}}]$.
\end{definition}

We now prove that for a non-empty set $\sW$ of variables, Halpern's definition matches our own.

\begin{theorem}
For a non-empty set of variables $\sW$,
the presence of additive Intent is equivalent to an agent intending to affect $\incentivevars$ in Halpern's definition.
\end{theorem}

\begin{proof}
    \emph{Proof that subtractive intent implies Halpern intent}
    If there is additive intent over every $\spi' \in \sPi'$, then 
    $\EE[\totutilvar_\spi] < \EE[\totutilvar_{\spi',\sZ_{\spi}}]$ for every $\spi \in \sPi'$, and so
    $\EE[\totutilvar_\spi] < \max_{\spi'} \EE[\totutilvar_{\spi',\sZ_{\spi}}]$, implying Halpern intent.
    \emph{Proof that Halpern intent implies additive intent}
    To begin with, if $\EE_{\spi'}[\totutilvar] \geq \EE_{\spi^*}[\totutilvar]$, 
    then we would have that $\sZ = \emptyset$ would always satisfy (a), and so there could not exist any 
    non-empty set $\sW$ satisfying Halpern intent.
    Since $\sW$ is assumed to be non-empty, we must therefore have $\EE_{\spi'}[\totutilvar] < \EE_{\spi^*}[\totutilvar]$, 
    satisfying the first condition of additive intent.
    Moreover, if $\EE[\totutilvar_\spi] < \max_{\spi'} \EE[\totutilvar_{\spi',\sZ_{\spi}}]$
    we have $\EE[\totutilvar_\spi] < \EE[\totutilvar_{\spi',\sZ_{\spi}}]$ for every $\spi \in \sPi'$, 
    satisfying the other condition, meaning that additive intent is present.
\end{proof}

\section{Proofs} \label{app:ai-acp:proofs}
\subsection{Preliminaries}
\label{app:ai-acp:appendix_prelims}

Our proofs will rely on the following fundamental results about causal models from
\citep{correa2020calculus},
\citep{galles1997axioms} and \citep{pearl2009causality}.~\looseness=-1

\begin{definition}[Causal Irrelevance]\label{def:ai-acp:causal-irrelevance}
$\sX$ is \emph{causally irrelevant} to $\sY$, given $\sZ$, written
$(\sX \irrelevant \sY | \sZ)$ if, for every set $\sW$ disjoint of
$\sX \cup \sY \cup \sZ$, we have
\begin{align*}
  \forall \eps, \sz, \sx, \sx', \sw \qquad
  \sY_{\sx\sz\sw}(\eps) &= \sY_{\sx'\sz\sw}(\eps)
\end{align*}
\end{definition}
 \begin{lemma}\label{theorem:path-causal-irrelevance}
Recall that $(\sX \nopathto \sY | \sZ)_\causalgraph$ means that $\causalgraph$ contains no 
directed path from $\sX$ to $\sY$, except possibly through $\sZ$.
Then, for every SCM $\causalmodel$ compatible with a DAG $\causalgraph$,
\begin{align*}
  {(\sX \nopathto \sY | \sZ)}_\causalgraph \Rightarrow
  (\sX \irrelevant \sY | \sZ)
\end{align*}
\end{lemma}
 \begin{proof*}
By induction over variables, as in \cite[Lemma~12]{galles1997axioms}.
\end{proof*}

\begin{lemma}[{\citealp[Thm. 3.4.1, Rule 1]{pearl2009causality}}]\label{theorem:do-calc-insertion-of-obs}
For any disjoint subsets of variables $\sW, \sX, \sY, \sZ$ in the DAG
$\causalgraph$,
$\EE(\sY_{\sx} | \sz, \sw) = \EE(\sY_{\sx} | \sw)$
if ${\sY \dsepag{\causalgraph'} \sZ | (\sX, \sW)}$ in the graph $\causalgraph'$
formed by deleting all incoming edges to $\sX$.
\end{lemma}
 \begin{lemma}[{\citealp[Thm. 1.2.4]{pearl2009causality}}]\label{theorem:d-separation-independence}
For any three disjoint subsets of nodes $(\sX, \sY, \sZ)$ in a DAG
$\causalgraph$, $(\sX \dsep_\causalgraph \sY | \sZ)$ if and only if
${(\sX \indep \sY | \sZ)}_P$ for every probability function $P$ compatible with $\causalgraph$.
\end{lemma}
 \begin{lemma}[{\citealp[Sigma Calculus Rule 3]{correa2020calculus}}]\label{le:sigma-calculus-intervention}
For any disjoint subsets of nodes $(\sX, \sY) \subseteq \sV$ and $\sZ \subseteq \sV$ in a DAG $\causalgraph$
$\Pr(\sX | \sZ;g^Y) = \Pr(\sX|\sZ;g'^Y)$
if $\sX \dsep \sY | \sZ$ in $\causalgraph_{\overline{\sY(\sZ)}}$
where $\sY(\sZ) \subseteq \sY$ is the set of elements in $\sY$ that are not ancestors of $\sZ$ in $\causalgraph$
and $\causalgraph_{\overline{\sW}}$ denotes $\causalgraph$ but with edges incoming to variables in $\sW$ removed.\end{lemma}

\subsection{An optimal policy that respects the minimal reduction}\label{app:ai-acp:appendix-minimal-reduction}
\newcommand*{\Padg}{\Pad_\causalgraph}
\newcommand*{\padg}{\pad_\causalgraph}

\newcommand*{\Padr}{\ensuremath{\Pad_\text{min}}}
\newcommand*{\padr}{\ensuremath{\pad_\text{min}}}
\newcommand*{\Padnr}{\ensuremath{\Pad_-}}
\newcommand*{\padnr}{\ensuremath{\pad_-}}
\newcommand*{\tpadnr}{\ensuremath{\varidx{\tilde{\pa}}{\decisionvar}_-}}

\newcommand*{\utildvars}{\utilvarsd}
\newcommand*{\utilndvars}{\utilvarsv{\setminus\decisionvar}}
\newcommand*{\tudvar}{\varidx{\totutilvar}{\decisionvar}}
\newcommand*{\tundvar}{\varidx{\totutilvar}{\setminus\decisionvar}}

First, we introduce the notion of a $\rcid$-respecting optimal policy. Our proof of its
optimality is similar to Theorem 3 from \citep{lauritzen2001representing}.
It builds on the following intersection property of d-separation.
\begin{lemma}[d-separation \emph{intersection} property]
  \label{le:d-sep-intersection-property}
  For all disjoint sets of variables $\sW$, $\sX$, $\sY$, and $\sZ$,
  \begin{equation*}
    (\sW \dsepg \sX | \sY, \sZ) \land
    (\sW \dsepg \sY | \sX, \sZ) \Rightarrow
    (\sW \dsepg (\sX \cup \sY) | \sZ)
  \end{equation*}
\end{lemma}
 \begin{proof}
Suppose that the RHS is false, so there is a path from $\sW$ to $\sX \cup
\sY$ conditional on $\sZ$.
This path must have a sub-path that passes from $\sW$ to $X \in \sX$ without
passing through $\sY$ or to $Y \in \sY$ without passing through $\sX$ (it must
traverse one set first).
But this implies that $\sW$ is d-connected to $\sX$ given $\sY,\sZ$ or to $\sY$
given $\sX,\sZ$, meaning the LHS is false.
So if the LHS is true, then the RHS must be true.
\end{proof}

\begin{lemma}[$\rcid$-respecting optimal policy] \label{le:reduced-optimal-policy}
    Every single-decision SCIM $\scim = \scimdef$ has an optimal policy
    $\tilde\pi$ that depends only on requisite observations.
    In other words, $\tilde\pi$ is also a policy for the minimal model
    $\rscim = {\left\langle \rcid, \exovars, \structfns, \exoprob \right\rangle}$.
    We call $\tilde\pi$ a \emph{$\rcid$-respecting optimal policy}.
\end{lemma}

This result is already known from \citep{lauritzen2001representing,fagiuoli1998note}, but we prove it here to make the paper more self-contained.

\begin{proof}
First partition $\Padg$ into the requisite parents
$\Padr = \{W \in \Pad:W \not \dsepg \utildvars \mid \{\decisionvar\} \cup \Pad \setminus \{W\}\}$, and non-requisite parents $\Padnr = \Padg \setminus \Padr$.

Let $\pi^*$ be an optimal policy in $\scim$.
To construct a $\rcid$-respecting version $\tilde\pi$,
select any value $\tpadnr \in \dom(\Padnr)$ for which $\Prs{\pi^*}(\Padnr = \tpadnr) > 0$.
For all $\padr\in\dom(\Padr)$ and $\exovalv{\decisionvar}\in\dom(\exovarv{\decisionvar})$, let
\begin{align*}
  \tilde\pi (\padr, \padnr, \exovalv{\decisionvar})
  \coloneqq \pi^*(\padr ,\tpadnr, \exovalv{\decisionvar}).
\end{align*}

The policy $\tilde\pi$ is permitted in $\rscim$ because it does not vary with $\Padnr$.

Now let us prove that $\tilde \pi$ that is optimal in $\scim$.
Partition $\utilvar$ into $\utildvars = \utilvars \cap \Descv{\decisionvar}$
and $\utilndvars = \utilvars \setminus \Descv{\decisionvar}$.
$\decisionvar$ is causally irrelevant for every $\utilvar \in \utilndvars$
so every policy $\pi$ (in particular, $\tilde\pi$) is optimal with respect to
$\tundvar \coloneqq \sum_{\utilvar \in \utilndvars}{\utilvar}$.

We now consider $\utildvars$.
By definition, $W \dsepg \utildvars \mid \{\decisionvar\} \cup \Pad \setminus \{W\}$
for every ${W \in \Padnr}$.
By inductively applying the intersection property of d-separation (\cref{le:d-sep-intersection-property}) over elements
of $\Padnr$ we obtain
\begin{equation}\label{eq:padx-indep-tudvar}
    \Padnr \dsep \utildvars \mid \{\decisionvar\} \cup \Padr.
\end{equation}

Next, we establish that
$\EEs{\tilde\pi}[\tudvar] = \EEs{\pi^*}[\tudvar]$ by showing that
${\EEs{\tilde\pi}[\tudvar \mid \pad]} \!=\!
{\EEs{\pi^*}[\tudvar \mid \pad]}$ for every
$\pad \in \dom(\Pad)$
with $\Pr(\pad) > 0$.
First, the expected utility of $\tilde\pi$ given any $(\padr, \padnr)$
with $\Pr({\Padr \!=\! \padr}, {\padnr \!=\! \padnr}) > 0$ is equal
to the expected utility of $\pi^*$ on input $(\padr, \tpadnr)$:
\begin{align*}
  \mathrlap{\EEs{\tilde\pi}[\tudvar \mid \padr, \padnr]}
  \hspace{0.5cm} &
  \\
  &= \sum_{u, \decisionval}{\begin{aligned}[t]\Big(
    &u
    \Pr(\tudvar = u \mid \decisionval, \padr, \padnr) \\
    &\cdot \Prs{\tilde\pi}(\decisionvar = d \mid \padr, \padnr)
  \Big)\end{aligned}}
  \\
  &
  = \sum_{u, \decisionval}{\begin{aligned}[t]\Big(
    &u
    \Pr(\tudvar = u \mid \decisionval, \padr, \tpadnr) \\
    &\cdot \Prs{\pi^*}(\decisionvar = \decisionval \mid \padr, \tpadnr)
  \Big)\end{aligned}}
  \\
  &
    = \EEs{\pi^*}[\tudvar \mid \padr, \tpadnr]
    \intertext{where the middle equality follows from \cref{eq:padx-indep-tudvar} and the
  definition of $\tilde\pi$.
  Second, the expected utility of $\pi^*$ given input $\tpadnr$ is the same as
    its expected utility on any input $\padnr$:
}
  &
  = \max_{\decisionval}{
    \EEs{\pi^*}[\tudvar_\decisionval \mid \padr, \tpadnr]
  }
  \\
  &
  = \max_{\decisionval}{
    \EEs{\pi^*}[\tudvar_\decisionval \mid \padr, \padnr]
  }
  \\
  &
  = \EEs{\pi^*}[\tudvar \mid \padr, \padnr ]
\end{align*}
where the first equality follows from the optimality of $\pi^*$ and
the second from \cref{theorem:do-calc-insertion-of-obs}.
The expression $\EEs{\pi^*}[\tudvar_\decisionval \mid\cdots]$ means that we
first assign the policy $\pi^*$ then intervene to set
$\decisionvar = \decisionval$, which renders $\pi^*$ effectively irrelevant but
formally necessary for creating an SCM.
This result shows that $\tilde\pi$ is optimal for $\tudvar$ and has
${\EEs{\tilde\pi}[\tudvar]} = {\EEs{\pi^*}[\tudvar]}$.
Since $\tilde\pi$ is optimal for both $\totutilvard$ and $\tundvar$, $\tilde\pi$ is optimal in $\scim$.
\end{proof}

\subsection{Response Incentive Criterion} \label{app:ai-acp:appendix-ri}
We now prove the soundness and completeness of the
response incentive criterion.

\ResponseIncentiveCriterion*

\begin{proof}[Proof of \Cref{theorem:ri-graph-criterion}]
  We first prove that the criterion is sound, and then that it is complete.
  
\textbf{Soundness (the \emph{only if} direction)}.
For the soundness direction,
assume that for $\cid$, the minimal reduction $\rcid$ contains no directed path $\incentivevar \pathto \decisionvar$ 
for any $\incentivevar \in \incentivevars$.
Let $\scim = \scimdef$ be any SCIM compatible with $\cid$.
Let $\rscim = \left\langle \rcid,\exovars,\structfns,\exoprob \right\rangle$
be $\scim$, but with the minimal reduction $\rcid$.
By \cref{le:reduced-optimal-policy} in \cref{app:ai-acp:proofs}, there exists a $\rcid$-respecting policy $\tilde \pi$ that is optimal in $\scim$.
In $\rscim_{\tilde \pi}$,
$\incentivevars$ is causally irrelevant for $\decisionvar$, so $\decisionvar(\exovals) = \decisionvar_{g^\incentivevars}(\exovals)$.
Furthermore, $\scim_{\tilde \pi}$ and $\rscim_{\tilde \pi}$ are the same SCM,
with the functions $\structfns \cup \{\tilde \pi\}$.
So $\decisionvar(\exovals) = \decisionvar_{g^\incentivevars}(\exovals)$ also in
$\scim_{\tilde\pi}$, which means that there is an optimal policy in $\scim$ that
does not respond to interventions on $\incentivevars$ for any $\exovals$.

\textbf{Completeness (the \emph{if} direction).} 
\Cref{fig:ri-completeness-2} illustrates the model constructed in the proof.


\newcommand*{\pwu}{{\overline{\riparvar\utilvar}}}
\newcommand*{\pdu}{{\overrightarrow{\decisionvar\utilvar}}}
\newcommand*{\pxd}{\overrightarrow{\incentivevar\decisionvar}}
\newcommand*{\wusrc}{S}
\newcommand*{\wucol}{C}
\newcommand*{\wuobs}{O}
\newcommand*{\pco}[1]{\ensuremath{\overrightarrow{\wucol^{#1}\wuobs^{#1}}}}
\newcommand*{\riinter}{\ensuremath{\incentivevar = 0}}

\begin{figure}[H]
  \centering
    \begin{influence-diagram}[
      every node/.append style = {circle},
      node distance=9mm
  ]
    \node[] (Sl) {$\wusrc^m$};
    \node[below right = of Sl] (Olp) {$\wucol^{m}$};
    \node[right = of Olp] (Ol) {$\wuobs^{m}$};
    \node[below = of Olp, yshift=-5mm] (\wuobs1p) {$\wucol^{1}$};
    \node[right = of O1p] (O1) {$\wuobs^{1}$};
    \node[below left = of O1p] (S0) {$\wusrc^0$};
    \node[below right = of S0, draw=none] (Oh)  {};
    \node[right = of Oh] (O) {$\riparvar$};
    \node[left = of O] (Z) {$Z$};
    \node[below left = of Z] (X) {$\incentivevar$};

    \node[draw=none] (sd) at ($(Sl)!0.5!(S0)$)  {$\rvdots$};
    \node[draw=none] at ($(Ol)!0.5!(O1)$)  {$\rvdots$};
    \node[draw=none] at ($(Olp)!0.5!(O1p)$)  {$\rvdots$};

    \node[decision, right = of O1] (A)  {$\decisionvar$};
    \node[right = of Sl, xshift=4cm] (Up)  {$Y$};
    \node[right = of Up, utility] (U)  {$\utilvar$};

    \path[dashed, ultra thick]
    (Sl) edge[->] (Up)
    (Sl) edge[->] (Olp)
    (S0) edge[->] (O1p)
    (S0) edge[->] (Z)
    (Z) edge[->] (O)
    (X) edge[->] (Z)
    (A) edge[->] (Up)
    (Up) edge[->] (U)
    (sd.south) edge[->] (O1p)
    (sd.north) edge[->] (Olp)
    ;

    \path[dashed]
    (Olp) edge[->] (Ol)
    (O1p) edge[->] (O1)
    ;

    \path
    (Ol) edge[->, information] (A)
    (O1) edge[->, information] (A)
    (O) edge[->, information] (A)
    ;

    \begin{scope}[
      node distance = 1mm,
      every node/.style = {rectangle, draw=none},
      ]
    \node[above = of Up] {\small $Y=\wusrc^m \cdot D$};
    \node[below = of O1]  {\small $\wuobs^1=\wucol^1$};
    \node[above = of Ol] {\small $\wuobs^m=\wucol^m$};
    \node[below = of U] {\small $\utilvar=Y$};
    \node[below = of O] {\small $\riparvar = Z$};
    \node[left = of Z] {\small $Z = \wusrc^0 \cdot \incentivevar$};
    \node[left = of X] {\small $\incentivevar = 1$};
    \node[left = of Olp] {\small $\wucol^m = \wusrc^{m-1}\cdot \wusrc^{m}$};
    \node[left = of Sl] {\small $\wusrc^m\sim \textrm{Uniform}(\{-1,1\})$};
    \node[left = of S0] {\small $\wusrc^0\sim \textrm{Uniform}(\{-1,1\})$};
    \node[left = of O1p] {\small $\wucol^1 = \wusrc^0\cdot \wusrc^1$};
    \node[right = of A] {\small choose $\decisionvar \in \{-1, 0, 1\}$};
    \end{scope}

  \end{influence-diagram}
  \caption{
    Outline of the variables involved in the response incentive construction.
    Every graph that satisfies the response incentive graphical criterion
    contains this structure (allowing all dashed paths except
    those to $\wucol^i$ or $Y$ to have length zero).
    An optimal policy for the given model is
    $D = \riparvar \cdot \prod_i \wuobs^i = \wusrc^m$,
    yielding utility $U = Y = \incentivevar(\wusrc^m)^2 = 1$,
    and all optimal policies must depend on the value of $\riparvar$.
  }\label{fig:ri-completeness-2}
\end{figure}
Starting from the assumption that there exists $\incentivevar \in \incentivevars$ with 
$\incentivevar \pathto \decisionvar$ in
$\rcid$, we explicitly construct a compatible model for $\cid$
for which the decision of every optimal policy causally depends on the value of
$\incentivevar$.
Let $\pxd$ be a directed path from $\incentivevar$ to $\decisionvar$
that only contains a single requisite observation that we label
$\riparvar$ (if $\incentivevar$ is itself a requisite observation, then $\incentivevar$ and $\riparvar$ are the
same node).
Since $\riparvar$ is a requisite observation for $\decisionvar$, there exists
some utility node $\utilvar$ descending from $\decisionvar$ that is d-connected
to $\riparvar$ in $\cid$ when conditioning on $\reqobscond$.
Let $\pdu$ be a directed path from $\decisionvar$ to $\utilvar$ and
let $\pwu$ be a path between $\riparvar$ and $\utilvar$ that is active
when conditioning on $\reqobscond$.
By the definition of d-connecting paths,
$\pwu$ has the following structure ($m \ge 0$):
\begin{center}
\begin{tikzpicture}[scale=0.9]
\node at (0, 0) (W) {$\riparvar$};
  \node at (1, 1) (S0) {$\wusrc^0$};
  \draw (W) edge[<-, dashed] (S0);
  \node at (2, 0) (O1) {$\wucol^1$};
  \draw (S0) edge[->, dashed] (O1);
  \node at (3, 1) (S1) {$\wusrc^1$};
  \draw (O1) edge[<-, dashed] (S1);
  \node[minimum height = 1.5em, minimum width = 2em, draw=none] at (4, 0)
    (O2) {};
  \draw (S1) edge[->, dashed] (O2);
  \node at (4.5, 0.5) {$\cdots$};
  \node[minimum height = 1.5em, minimum width = 2em, draw=none] at (5, 1)
    (Sm1) {};
  \node at (6, 0) (Om) {$\wucol^m$};
  \draw (Sm1) edge[->, dashed] (Om);
  \node at (7, 1) (Sm) {$\wusrc^m$};
  \draw (Om) edge[<-, dashed] (Sm);
  \node at (8, 0) (U) {$\utilvar$};
  \draw (Sm) edge[->, dashed] (U);
\end{tikzpicture}
\end{center}
consisting of directed sub-paths leaving source nodes $\wusrc^i$ and entering
collider nodes $\wucol^i$, where there is a directed path from each collider to
$\reqobscond$ and no non-collider node is in $\reqobscond$.
It may be the case that $\riparvar$ and $\wusrc^0$ are the same node.
For each $i \in \{1, \ldots, m\}$, let $\pco{i}$ be a directed path from
$\wucol^i$ to some $\wuobs^i \in \Pad$ such that no other node along $\pco{i}$
is in $\Pad$.

We make the following assumptions without loss of generality:
\begin{itemize}
  \item $\pwu$ first intersects $\pdu$ at some variable $Y$ (possibly $Y$
    is $\utilvar$) and thereafter both $\pwu$ and $\pdu$ follow the same directed
    path from $Y$ to $U$ (otherwise, let $Y$ be the first intersection point and
    replace the $Y \pathto U$ sub-path of $\pwu$ with the $Y \pathto U$ sub-path
    of $\pdu$).
  \item The $\wusrc^0 \pathto \riparvar$ sub-path of reversed $\pwu$ first intersects $\pxd$
    at some node $Z$ and thereafter both follow the same directed path from $Z$
    to $\riparvar$ (same argument as for $Y$).
  \item The paths $\pco{i}$ are mutually non-intersecting
    (if there is an intersection between $\pco{i}$ and $\pco{j}$ with $j \ne i$
    then replace the part of $\pwu$ between $\wucol^i$ and $\wucol^j$ with the
    path through the intersection point, which becomes the new collider; this
    can only happen finitely many times as it reduces the number of collider
    nodes).
\end{itemize}
The resulting structure is shown in \cref{fig:ri-completeness-2}.

We now formally define the model represented in the figure.
The domains of all endogenous variables are set to $\{-1, 0, 1\}$.
All exogenous variables are given independent discrete uniform
distributions over $\{-1, 1\}$.
Unless otherwise specified, we set $B = A$ for each edge $A \to B$ within the
directed paths shown in \cref{fig:ri-completeness-2},
i.e. $\fv{B}(\pav{B}, \exovalv{B}) = a$.
Nodes at the heads of directed paths can therefore be defined in terms of nodes
at the tails.
We begin by describing functions for the ``default'' case depicted by
\cref{fig:ri-completeness-2}, and discuss adaptations for various special cases
below.
\begin{itemize}
\item $\wusrc^i = \exovarv{\wusrc^i}$, giving
  $\wusrc^i$ a uniform distribution over $-1$ and $1$.
\item $U = Y$, and
\item $Y = \wusrc^m \cdot D$, so $D$ must match $S^m$ to optimize utility.
\item $\wucol^i = \wusrc^{i - 1} \cdot \wusrc^i$, and
\item $\wuobs^i = \wucol^i$, so the collider $\wucol^i$
  reveals (only) whether $\wusrc^{i - 1}$ and $\wusrc^i$ have the same sign or
  not.
\item $\incentivevar=1$,
\item $Z=\incentivevar\cdot S^0$, and
\item $\riparvar=Z$, so $\riparvar$ reflects the value of $S^0$, unless $\incentivevar$ is intervened upon.
\end{itemize}
All other variables not part of any named path are set to $0$.

Special cases arise when two or more of the labeled nodes in 
\cref{fig:ri-completeness-2} refer to the same variable.
When $\riparvar$, $Y$, or $O^i$ is the same node as one of its parents,
then it simply takes the function of this parent (instead of copying its value).
Meanwhile, the $S^i$, $C^i$, and $Y$ nodes must be distinct by construction, so
no special cases treatment is required.
Finally, the functions for $\incentivevar$, $S^0$ and $Z$ are adapted per the following cases:

\emph{Case 1: $\incentivevar$, $S^0$, and $Z$ are all the same node}.
Let
$\incentivevar=Z=S^0=\exovarv{S^0}$, i.e.\ the node takes a uniform distribution over
$\{-1, 1\}$.

\emph{Case 2: $Z$ is the same node as $S^0$, but different from $\incentivevar$}.
In this case, let $Z=S^0=\incentivevar\cdot \exovarv{S^0}$.

\emph{Case 3: $\incentivevar$ is the same node as $Z$, but different from $S^0$}.
In this case, let $\incentivevar=Z=S^0$.

The final combination of $\incentivevar$ and $S^0$ being the same, while different from $Z$,
cannot happen by the definition of $Z$.

Regardless of which case applies, an optimal policy is
$\decisionvar = \riparvar \cdot \prod_{i=1}^m{\wuobs^i}$,
which yields a utility of $1$.

Let $g^{\incentivevars}$ be the intervention $\doo(W=0)$.
Formally, $g^{\incentivevars}$ has $g^\incentivevar$
deterministically set ${\incentivevar = 0}$, 
and applies the unchanged function 
$g^{\incentivevar'}=f^{\incentivevar'}$
for the other variables $\incentivevar' \in \incentivevars \setminus \{\incentivevar\}$.
Under $g^{\incentivevars}$, it follows that $\riparvar_{\riinter} = Z_{\riinter} = 0$.
Without the information in $\riparvar$, $\wusrc^m$ is independent of
${(\Pad)}_{\riinter}$ and hence is independent of $\decisionvar_{\riinter}$
regardless of the selected policy.\footnote{
  Note that if $m = 0$ and $\wusrc^0$ is $Z$ then
  ${(\wusrc^m)}_{\riinter} = 0$ but the fact that this is predictable is
  irrelevant because we compare
  $\decisionvar_{\riinter}$ against the pre-intervention variable $\wusrc^m$,
  which remains independent of ${(\Pad)}_{\riinter}$.}
Therefore,
$\EEs{\pi}[\utilvar_{\decisionvar_{\riinter}}]
= \EEs{\pi}[\wusrc^m \cdot \decisionvar_{\riinter}]
= \EEs{\pi}[\wusrc^m] \cdot \EEs{\pi}[\decisionvar_{\riinter}]
= 0$
for every
policy $\pi$.
In particular, for any optimal policy $\pi^*$,
$\EEs{\pi^*}[\utilvar_{\decisionvar_{\riinter}}] \ne
\EEs{\pi^*}[\utilvar] = 1$.
Thus, there must be some $\exovals$ such that
$\decisionvar_{\riinter}(\exovals) \ne \decisionvar(\exovals)$.
And by the definition of $g^\incentivevars$, we have that
$\decisionvar_{g^\incentivevars}(\exovals)=\decisionvar_{\riinter}(\exovals)$, 
so there is a response incentive on $\incentivevar$.
\end{proof}

\subsection{Materiality Criterion} \label{app:ai-acp:appendix-materiality}

We begin by restating the graphical criterion for materiality.

\MaterialityCriterion*

The proof is as follows.

\begin{proof}
\textbf{Soundness.}
Assume that $V$ is nonrequisite for $D$.
There always exists an optimal policy that respects $\rcid$ (\cref{le:reduced-optimal-policy})
and this policy 
is also permitted in $\calM_{V \not \to D}$ 
since $\rcid$ does not contain $V \not \to D$,
so this policy achieves $\attutil(\calM)$,
proving that $V$ is immaterial.

\textbf{Completeness.}
By assumption, $W \dsep U \mid \Pa{D} \setminus \{W\}$.
So, we construct the same model as in 
the proof of \Cref{theorem:ri-graph-criterion}, 
for the special case where $W$ and $X$ are the same node, 
as shown in \cref{fig:materiality-completeness}.

Clearly the policy $D=X \cdot \prod_{i=1}^m O^i$ 
still yields a utility of $1$, which is optimal.

Let $\pa^D_{\setminus W}$ be an arbitrary assignment to the parents of 
$D$ except $W$, and let $s^m$ be an assignment to $S^m$.
Notice that if we have an assignment,
$S^m=s^m,\Pa^D_{\setminus W}=\pa^D_{\setminus W}$, this uniquely identifies 
an assignment to the variables $S^0:m$, because 
$S^{i-1} = S^i \oplus O^i$ for every $1\leq i \leq m$.
It follows that: 

$P(S^m=1,\pa^D_{\setminus W}) = P(S^m=-1,\pa^D_{\setminus W}) = \frac{1}{2}.$

and so, for any observations $\pa^D_{\setminus W}$, we have

$$P(S^m=1 \mid \pa^D_{\setminus W}) = P(S^m=-1 \mid \pa^D_{\setminus W}) = \frac{1}{2}.$$

Therefore, a deterministic policy can map each $\pa^D_{\setminus W}$, 
to $1$ or $-1$, 
in which case we will have 
$P(U=1 \mid \pa^D_{\setminus W})=P(U=-1 \mid \pa^D_{\setminus W})=\frac{1}{2}$, 
and so $\EE[U \mid \pa^D_{\setminus W}]=1$, 
or to $o$, in which case $U=0$ always.

It follows that marginalising across every $\pa^D_{\setminus W}$, any 
deterministic policy will obtain $\EE[U]=0$.

Furthermore, the best stochastic policy never outperforms the best deterministic policy
\citep[Proposition 1]{lee2020characterizing}.

Hence, the attainable utility when $W$ is not observed is $0$, 
whereas the attainable utility when $W$ is observed is $1$, proving the result.

\newcommand*{\pwu}{{\overline{\riparvar\utilvar}}}
\newcommand*{\pdu}{{\overrightarrow{\decisionvar\utilvar}}}
\newcommand*{\pxd}{\overrightarrow{\incentivevar\decisionvar}}
\newcommand*{\wusrc}{S}
\newcommand*{\wucol}{C}
\newcommand*{\wuobs}{O}
\newcommand*{\pco}[1]{\ensuremath{\overrightarrow{\wucol^{#1}\wuobs^{#1}}}}
\newcommand*{\riinter}{\ensuremath{\incentivevar = 0}}
\begin{figure}[H]
  \centering
    \begin{influence-diagram}[
      every node/.append style = {circle},
      node distance=9mm
  ]
    \node[] (Sl) {$\wusrc^m$};
    \node[below right = of Sl] (Olp) {$\wucol^{m}$};
    \node[right = of Olp] (Ol) {$\wuobs^{m}$};
    \node[below = of Olp, yshift=-5mm] (\wuobs1p) {$\wucol^{1}$};
    \node[right = of O1p] (O1) {$\wuobs^{1}$};
    \node[below left = of O1p] (S0) {$\wusrc^0$};
    \node[below right = 0mm and 25mm of S0] (O) {$\incentivevar$};

    \node[draw=none] (sd) at ($(Sl)!0.5!(S0)$)  {$\rvdots$};
    \node[draw=none] at ($(Ol)!0.5!(O1)$)  {$\rvdots$};
    \node[draw=none] at ($(Olp)!0.5!(O1p)$)  {$\rvdots$};

    \node[decision, right = of O1] (A)  {$\decisionvar$};
    \node[right = of Sl, xshift=4cm] (Up)  {$Y$};
    \node[right = of Up, utility] (U)  {$\utilvar$};

    \path[dashed, ultra thick]
    (Sl) edge[->] (Up)
    (Sl) edge[->] (Olp)
    (S0) edge[->] (O1p)
    (S0) edge[->] (O)
    (A) edge[->] (Up)
    (Up) edge[->] (U)
    (sd.south) edge[->] (O1p)
    (sd.north) edge[->] (Olp)
    ;

    \path[dashed]
    (Olp) edge[->] (Ol)
    (O1p) edge[->] (O1)
    ;

    \path
    (Ol) edge[->, information] (A)
    (O1) edge[->, information] (A)
    (O) edge[->, information] (A)
    ;

    \begin{scope}[
      node distance = 1mm,
      every node/.style = {rectangle, draw=none},
      ]
    \node[above = of Up] {\small $Y=\wusrc^m \cdot D$};
    \node[below = of O1]  {\small $\wuobs^1=\wucol^1$};
    \node[above = of Ol] {\small $\wuobs^m=\wucol^m$};
    \node[below = of U] {\small $\utilvar=Y$};
    \node[below = of O] {\small $\incentivevar = \wusrc^0$};
    \node[left = of Olp] {\small $\wucol^m = \wusrc^{m-1}\cdot \wusrc^{m}$};
    \node[left = of Sl] {\small $\wusrc^m\sim \textrm{Uniform}(\{-1,1\})$};
    \node[left = of S0] {\small $\wusrc^0\sim \textrm{Uniform}(\{-1,1\})$};
    \node[left = of O1p] {\small $\wucol^1 = \wusrc^0\cdot \wusrc^1$};
    \node[right = of A] {\small choose $\decisionvar \in \{-1, 0, 1\}$};
    \end{scope}

  \end{influence-diagram}
  \caption{
    The materiality construction.
}\label{fig:materiality-completeness}
\end{figure}
\end{proof}

\subsection{VoI Criterion} \label{app:ai-acp:appendix-voi}

We begin by restating the graphical criterion for value of information. 

\ValueOfInformationCriterion*
 
\begin{proof}
Let $\calM_{W \to D}$ be a SCIM
identical to $\calM$ except that an edge is added from $W \to D$
if one is not present already, and let $\calG_{W \to D}$ be 
its associated graph.
Notice that positive value of information in $\calM$ 
and materiality in $\calM_{W \to D}$, 
are both equivalent to
$\attutil(\calM_{W \not \to D}) < \attutil(\calM_{W \to D})$.
So, $\calG$ is compatible with positive 
value of information 
precisely when $\calG_{W \to D}$ 
is compatible with materiality, 
i.e.\ when $W \not \perp U(D) \mid \Pa^D \setminus W$ in $\calG_{W \to D}$.
\end{proof}

\subsection{Instrumental Control Incentive Criterion} \label{app:ai-acp:appendix-ici}
We first restate the ICI criterion.

\InstrumentalControlIncentiveCriterion*

The proof is as follows.

\begin{proof}We first prove the soundness direction, followed by the completeness direction.

\textbf{Soundness (the \emph{only if} direction)}.
Assume that there is no path $D \pathto W \pathto U$.
We will prove that the nested counterfactual has no effect,
\begin{equation}
{\totutilvar(\exovals)} =
    {\totutilvar_{\incentivevar_{\decisionval}}(\exovals).} \tag{*}
    \end{equation}
    and therefore that there is no instrumental control incentive.

Let $\scim$ be any SCIM compatible with $\cid$ and $\pi$ any policy for $\scim$.
Let $\sW' = \sW \cap \Descv{\decisionvar}$.
By \cref{theorem:path-causal-irrelevance}, 
$\sU_{\sW_\decisionval}(\exovals)=\sU_{\sW'_\decisionval}(\exovals)$ for all $\exovals$.
The variables $\sW'$ must be non-descendants of $\sU$ by assumption, 
so \cref{theorem:path-causal-irrelevance} implies that 
$\sU_{\sW'_\decisionval}(\exovals)=\sU(\exovals)$ for all $\exovals$. So (*) holds.




From (*), we have $\EEs{\pi}[\totutilvar \mid \pad]
= \EEs{\pi}[\totutilvar_{\incentivevars_\decisionval} \mid \pad]$,
so $\incentivevars$ has no ICI.

\textbf{Completeness (the \emph{if} direction).}
Assume that $\cid$ contains a directed path
$\decisionvar = Z^0 \to Z^1 \to \cdots \to Z^n = \utilvar$ where
$\utilvar \in \utilvars$ and
$Z^i \in \incentivevars$ for one or more $i \in \{0, \ldots, n\}$.
Let $j$ be the highest integer where $Z^j \in \incentivevars$, and 
note that $\incentivevars$ are assumed to be non-decisions, so we have $j > 0$.
We construct a compatible SCIM for which there is an instrumental control incentive
on $\incentivevars$, as well as additive and subtractive intent.
Let all variables along the path $Z^0 \to \ldots \to Z^n$ be
equal to their predecessor, except $Z^0 = \decisionvar$, which has no structural
function. All other variables are set to $0$.
In this model, $\utilvar \ceq \decisionvar \in \{0, 1\}$ and all other utility
variables are always $0$, so the only optimal policy is
$\pi^*(\pad) = 1$, which gives
${\EEs{\pi^*}[\totutilvar \mid \Pad=\bm{0}] = 1}$.
Meanwhile, $Z^j_{\decisionval=0}=0$, 
and under the intervention $Z^j=0$  this value is copied along to $\utilvar$, so
$\utilvar_{\incentivevars_{\decisionval}} = 0$, and hence
${\EEs{\pi^*}[\totutilvar_{\incentivevars_{\decisionval=0}} \mid
\Pad = \bm{0}] = 0}$, so there is an ICI.
\end{proof}

\subsection{Intent Criterion} \label{app:ai-acp:appendix-intent}
We begin by restating the graphical criterion for intent.
\IntentCriterion*

The proof is as follows.

\begin{proof}
We will first prove soundness, and then completeness.

\textbf{Soundness.}
As there is no path $D \pathto \incentivevar \pathto \utilvar$
for any $\incentivevar \in \incentivevars,\utilvar \in \utilvars$, 
equation (*) 
holds, by the same argument as in the proof of \cref{theorem:ici-graph-criterion}
(i.e.\ the nested counterfactual has no effect). 
We will then prove that there is:
 (a) no additive intent, and (b) no subtractive intent.

\emph{Proof of (a).}
Let us assume (*) and that additive intent is present, and we will prove a contradiction:
\begin{align*}
\EE_{\spi'}[\totutilvar_{\incentivevars_{\spi^*}}] &= \EE_\spi[\totutilvar] & (by (*)) \\
& < \EE_{\spi^*}[\totutilvar]&(\text{def. of intent} \\
& \leq \EE_{\spi'}[\totutilvar_{\incentivevars_{\spi^*}}], & (\text{\cref{eq:reason-to-move}}
\end{align*}
giving a contradiction. So it follows from (1) that there is no additive intent.

\emph{Proof of (b).}
Let us assume (*) and that subtractive intent is present and we will prove a contradiction:
\begin{align*}
\EE_{\spi^*}[\totutilvar_{\incentivevars_{\spi'}}] &= \EE_{\spi^*}[\totutilvar] & (by (*)) \\
& > \EE_{\spi'}[\totutilvar]& (\text{def. of intent}) \\
& \geq \EE_{\spi^*}[\totutilvar_{\incentivevars_{\spi^*}}] & (\text{\cref{eq:reason-not-to-move})}
\end{align*}
giving a contradiction. So there is no subtractive intent.

\textbf{Completeness.}
Consider the graph constructed in the proof of completeness 
for ICI (\cref{theorem:ici-graph-criterion}).
Letting ${\spi'}$ be the policy that chooses $D=0$, 
the same argument implies that
$0=\EE_{\spi'}[\totutilvar] <\EE_{\spi^*}[\totutilvar] = 1$ 
and $1=\EE_{\spi'}[\totutilvar_{\incentivevars_{\spi^*}}]\geq \EE_{\spi^*}[\totutilvar]=1$, 
which means that there is an additive intent to influence $\incentivevars$.
If we instead treat $\spi^*$ as the baseline policy and intervene ${\spi'}$, then by similar reasoning 
we have that
$0=\EE_{\spi'}[\totutilvar] <\EE_{\spi^*}[\totutilvar] = 1$ and
$0=\EE_{\spi^*}[\totutilvar_{\incentivevars_{\spi'}}] \leq \EE_{\spi'}[\totutilvar] = 0$, so there is subtractive intent.
\end{proof}\ryan{Is this proof explicit enough?}

\subsection{Impact Incentive Criterion} \label{app:ai-acp:appendix-ii}
We begin by restating the impact incentive criterion.
\ImpactIncentiveCriterion*

The proof is as follows.

\begin{proof} 
\textbf{Soundness.}
If $\incentivevars \cap \Desc(D)=\emptyset$, then by sigma calculus rule 3 \citep{correa2020calculus}, 
$\incentivevars_\spi(\seps)$ is invariant to $\spi$, and
$\incentivevar_{\spi}(\seps)=\incentivevar_{\spi'}(\seps)$, 
for all $\seps$.
Since $\delta$ is a distance function, 
it maps matching arguments to $0$,
so for any $c>0$, there is no impact incentive.
If $U \not \in \Desc(D)$, then similarly, $U$ is invariant to $\spi$, so every policy is optimal, and 
for any chosen baseline policy $\spi'$, there exists optimal $\spi=\spi'$, 
so as in the previous case,
 $\delta(\incentivevar_\spi(\seps),\incentivevar_{\spi'}(\seps))=0$ for all $\seps$, 
and there is no impact incentive.

\textbf{Completeness.}
By assumption, let $\calG$ be an arbitrary graph that contains the paths $X \pathfrom D$
and $\pathto U$ 
for some $\incentivevar \in \incentivevars$.
Then, define the model $\scim$ where $D \in \{0,1\}$ and the value of $D$ is copied along the paths to $\incentivevar$ and $U$, 
and all other variables are assigned a trivial domain.
To see that this yields in an impact incentive, 
note that to achieve $\mathbb{E}[U]=1$, any optimal policy $\spi$ must have 
$\incentivevar(\seps)=1$ for every $\seps$ with $P(\seps)>0$,
whereas the baseline policy $\spi'$ that always chooses $D=0$ has 
$\incentivevar(\seps)=0$ for all $\seps$.
Since $\delta$ is a distance measure, it follows that
$\delta(\incentivevar_\spi(\seps), \incentivevar_{\spi'}(\seps))>0$, 
and so there exists some $c$ for which there is an impact incentive.
\end{proof}

\subsection{Value of Control Criterion}\label{app:ai-acp:appendix-voc-criterion}

We first restate the criterion.

\ValueOfControlCriterion*

The proof is as follows.

\begin{proof}
\textbf{Soundness.}
    The proof of \emph{only if} (soundness) is as follows.
Let $\scim = \scimdef$ be a single-decision SCIM.
    Let $\scims{\gsw}$ be $\scim$, but with the structural functions $\fv{W}$ for $\incentivevar \in \incentivevars$
    replaced with $g^\incentivevar$.
    Let $\rscim$ and $\rscim_{\gsw}$ be the same SCIMs, respectively, but
    replacing each graph with the minimal reduction $\reducedgraph$.

    Recall that $ \EEs{\pi}[\totutilvar_{\gsw}]$ is defined by applying the soft
    interventions $\gsw$ to the (policy-completed) SCM $\scims{\pi}$.
    However, this is equivalent to applying the policy $\pi$ to the modified
    SCIM $\scims{\gsw}$, as the resulting SCMs are identical.
    Since $\scims{\gsw}$ is a SCIM, \cref{le:reduced-optimal-policy} can be applied,
    to find a $\reducedgraph$-respecting optimal policy $\tilde\pi$ for
    $\scims{\gsw}$.

    Consider now the expected utility under an arbitrary intervention $\gsw$
    for a policy $\pi$ optimal for $\scims{\gsw}$:
\begin{align*}
&\EEs{\pi}[\totutilvar_{\gsw}] \text{ in $\scim$} \\
      &=\EEs{\pi}[\totutilvar] \text{ in $\scims{\gsw}$} & \text{by SCM equivalence} \\
      &=\EEs{\tilde\pi}[\totutilvar] \text{ in $\scims{\gsw}$} & \text{by \cref{le:reduced-optimal-policy}}\\
      &=\EEs{\tilde\pi}[\totutilvar] \text{ in $\rscim_{\gsw}$} & \text{since $\tilde\pi$ is $\reducedgraph$-respecting} \\
        &=\EEs{\tilde\pi}[\totutilvar] \text{ in $\rscim$} & \text{by \cref{le:sigma-calculus-intervention}} \\
       &=\EEs{\tilde\pi}[\totutilvar] \text{ in $\scim$} & \text{only increasing the policy set} \\
       &\leq \max_{\pi^*}\EEs{\pi^*}[\totutilvar] \text{ in $\scim$} & \text{$\max$ dominates all elements.}
    \end{align*}
    This shows that $\incentivevars$ lack value of control.

\textbf{Completeness.}
Assume that $\incentivevar$ is an ancestor of some $\utilvar \in \utilvars$ for
some $\incentivevar \in \incentivevars$
and fix a particular directed path $\rho$ from $\incentivevar$ to some utility
$\utilvar \in \utilvars$.
We consider two cases depending on whether $\decisionvar$ is in $\rho$
and construct a SCIM for each:

\emph{Case 1: $\rho$ does not contain $\decisionvar$.}
Let the domain of all variables be $\{0, 1\}$.
Set all exogenous variable distributions arbitrarily.
Set $\structfns$ such that $\incentivevar = 0$ with every other variable along
$\rho$ copying the value of $\incentivevar$ forward. All remaining variables
are set to the constant $0$.
In this model, an intervention $g^\incentivevars$ that sets $\incentivevar$ to
$1$ instead of $0$, while assigning every other $\incentivevar' \in \incentivevars \setminus \{\incentivevar\}$
the unchanged function $g^{\incentivevar'}=f^{\incentivevar'}$,
increases the total expected utility by $1$, which means
there is an instrumental control incentive for $\incentivevar$.

\emph{Case 2: $\rho$ contains $\decisionvar$.}
This implies that a directed path $\incentivevar \to \decisionvar$ is present in
$\rcid$ so we can construct (a modified version of) the response incentive
construction used in the proof of completeness for
\cref{theorem:ri-graph-criterion}.
We make one change: instead of starting with
$\fv{\incentivevar}(\cdot) = 1$ we start with
$\fv{\incentivevar}(\cdot) = 0$.
As noted in the response incentive completeness proof, this means that
$S_m$ is independent of $\Pad$ so regardless of the policy the optimal
attainable utility is $0$.
If we perform the intervention $g^\incentivevars$
such that $\incentivevar=1$
and assign every other $\incentivevar' \in \incentivevars \setminus \{\incentivevar\}$
the unchanged function $g^{\incentivevar'}=f^{\incentivevar'}$
then the
expected utility is $1$ once again so the intervention $g^\incentivevars$
strictly increases the optimal expected utility.
\end{proof}

\subsection{Counterfactual Fairness}\label{app:ai-acp:appendix-fairness}
\theoremcffair*
\newcommand*{\supps}[1]{\ensuremath{\supp_{#1}}}
\begin{proof}
  We begin by showing that if there exists an optimal policy $\pi$ that is
  counterfactually fair, then there is no response incentive on $A$.
To this end, let
  \begin{align*}
    \supps{\pi}(D\mid \pad) &= \{ d\mid \Prs{\pi}(D = d\mid\pad)>0 \}\\
    \forall a,\  \supps{\pi}(D_a\mid \pad) &= \{ d\mid \Prs{\pi}(D_a = d\mid\pad)>0 \}
  \end{align*}
  be the sets of decisions taken by $\pi$ with positive probability with and
  without an intervention on $A$.
  As a first step, we will show that for any $\exovals\in\dom(\exovars)$ and
  any intervention $a$ on $A$,
  \begin{equation}
    \label{eq:supp}
    \supps{\pi}\big(D\mid \Pad(\exovals)\big)
    = \supps{\pi}\big(D_a\mid \Pad(\exovals)\big).
  \end{equation}
  By way of contradiction, suppose there exists a decision
  \begin{equation}
    \label{eq:set-diff}
    d\in
    \supps{\pi}\big(D\mid \Pad(\exovals)\big)\setminus
    \supps{\pi}\big(D_a \mid \Pad(\exovals)\big).
\end{equation}
  Since $d\in \supps{\pi}\big(D\mid \Pad(\exovals)\big)$,
  we have
  \begin{equation}
    \label{eq:pad-geq}
    \Prs{\pi}\left(D=d\mid \Pad(\exovals), A(\exovals)\right)> 0.
  \end{equation}
  And since $d\not\in\supps{\pi}\big(D_a\mid \Pad(\exovals)\big)$,
  there exists no $\exovals'$ with positive probability
  such that $\Pad(\exovals') = \Pad(\exovals)$,
  $A(\exovals') = A(\exovals)$,
  and $D_a(\exovals') = d$.
  This gives
  \begin{equation}
    \label{eq:pada-eq}
    \Prs{\pi}\left(D_{a}=d\mid \Pad(\exovals), A(\exovals)\right) = 0.
  \end{equation}
  \Cref{eq:pad-geq,eq:pada-eq} violate the counterfactual fairness property,
  \cref{eq:cf},
  which shows that \cref{eq:set-diff} is impossible.
  An analogous argument shows that
  $d\in \supps{\pi}\big(D_a\mid \Pad(\exovals)\big)\setminus
  \supps{\pi}\big(D\mid \Pad(\exovals)\big)$
  also violates the counterfactual fairness property \cref{eq:cf}.
  We have thereby established \cref{eq:supp}.

  Now select an arbitrary ordering of the elements of $\dom(D)$ and define a new
  policy $\pi^*$ such that $\pi^*(\pad)$ is the minimal element of
  $\supps{\pi}(D\mid \pad)$.
  Then $\pi^*$ is optimal because $\pi$ is optimal.
  Further, $\pi^*$ will make the same decision in
  decision contexts $\Pad(\exovals)$ and $\Pad_a(\exovals)$
  because of \cref{eq:supp}.
  In other words, $D_a(\exovals) = D(\exovals)$ in $\scim_{\pi^*}$ for the
  optimal policy $\pi^*$, which means that there is no response incentive on
  $\{A\}$.

  Now we prove the reverse direction --- that if there is no response incentive then some optimal $\pi^*$ is counterfactually fair.
  Choose any optimal policy $\pi^*$ where $D_a(\exovals)=D(\exovals)$ for all $\exovals$.
  Since an intervention $\do(A=a)$ cannot change $D$ in any setting, $\Pr(D_{a}=d\mid
  \cdot)=\Pr(D=d\mid \cdot)$ for any condition and any decision $d$, hence $\pi^*$ is counterfactually fair.
\end{proof}


\end{document}